\documentclass[draftcls,onecolumn,12pt]{IEEEtran}

\usepackage{amsmath,amstext,amssymb,epsf}
\usepackage{fullpage,latexsym}
\usepackage{epsfig}
\numberwithin{equation}{section} 

 \newtheorem{lemma}{Lemma}[section]
 \newtheorem{theorem}[lemma]{Theorem}

 \newtheorem{definition}[lemma]{Definition}
 \newtheorem{rem}[lemma]{Remark}
\newenvironment{remark}{\begin{rem}}{\hspace*{\fill}$\diamondsuit$\end{rem}}
 \newtheorem{ex}[lemma]{Example}

\newcommand{\NCD}{\textsc {NCD} }

\newcommand{\commentout}[1]{}
\renewcommand{\emptyset}{\varnothing}

\begin{document}

\title{A Fast Quartet Tree Heuristic for Hierarchical Clustering}
\author{Rudi Cilibrasi and Paul M.B. Vit\'{a}nyi
\thanks{
Rudi Cilibrasi is with the Center for Mathematics
and Computer Science (CWI). Address:
CWI, Science Park 123,
1098 XG Amsterdam, The Netherlands.
Email: {\tt cilibrar@gmail.com}.
 Part of his work was supported by
 the Netherlands BSIK/BRICKS project, and by NWO  project 612.55.002.
Paul Vit\'{a}nyi is with the Center for Mathematics and Computer Science (CWI),
and the University of Amsterdam.
Address:
CWI, Science Park 123,
1098 XG Amsterdam, The Netherlands.
Email: {\tt Paul.Vitanyi@cwi.nl}.
He was supported in part by the
EU project RESQ, IST--2001--37559, the NoE QUIPROCONE
IST--1999--29064,
the ESF QiT Programmme, and the EU NoE PASCAL II,
the Netherlands BSIK/BRICKS project.
}}

\maketitle

\begin{abstract}
The Minimum Quartet Tree Cost 
problem is to construct
an optimal weight tree
from the $3{n \choose 4}$ weighted quartet topologies on $n$ objects,
where optimality means that the summed weight of the embedded quartet topologies
is optimal (so it can be the case that the optimal tree embeds all quartets
as nonoptimal topologies). 
We present a Monte Carlo heuristic, based on randomized hill climbing,
for approximating the optimal weight tree, given the quartet topology weights.
The method
repeatedly transforms a dendrogram,
with all objects involved as leaves, achieving
a monotonic approximation to the exact single globally optimal tree.
The problem and the solution heuristic 
has been extensively used for general hierarchical clustering
of nontree-like (non-phylogeny) data
in various domains and across domains with heterogeneous data.
We also present a greatly improved heuristic, reducing the running time 
by a factor of order a thousand to ten thousand.
All this is implemented and available, as part of the CompLearn package.
We compare performance and running time of the original
and improved versions with those of UPGMA, BioNJ, and NJ,
as implemented in the SplitsTree package on genomic data for which the latter
are optimized.

{\em Keywords}---
Data and knowledge visualization,
Pattern matching--Clustering--Algorithms/Similarity measures, 
Pattern matching--Applications,

{\em Index Terms}---
hierarchical clustering,
global optimization,
Monte Carlo method,
quartet tree,
randomized hill-climbing,
\end{abstract}

\section{Introduction}
\label{sect.intro}
If we want to find structure in a collection of data, then we can organize
the data into clusters such that the data in the same cluster are similar
and the data in different clusters are dissimilar. 
In general
there is no best criterion to determine the clusters. One approach is to
let the user determine the criterion that suits his or hers needs.  
Alternatively, we can let the data itself 
determine ``natural'' clusters.
Since it is not likely that natural data determines 
unequivocal disjoint clusters,
it is common to hierarchically cluster the data \cite{DHS}. 

\subsection{Hierarchical Clustering and the Quartet Method}

In cluster analysis there are basically two methods for hierarchical
clustering. In the bottom-up approach initially every data item constitutes
its own cluster, and pairs of clusters are merged as one moves up the hierarchy.
In the top-down approach the set of all data constitutes the initial cluster, 
and splits are performed recursively as one moves down the hierarchy.
Generally, the merges and splits are determined in a greedy manner. 
The main disadvantages of the bottom-up and top-down methods 
are firstly that
they do not scale well because the time complexity 
is nonlinear in terms of the number of objects,
and secondly that they can 
never undo what was done before.
Thus, they are lacking in robustness and uniqueness 
since the results depend on the
earlier decisions. In contrast, the method which we propose here 
is robust and gives unique results in the limit.

The results of hierarchical clustering are usually 
presented in a dendrogram \cite{Jo67}. For a small number of
data items this has the added advantage that
the relations among the data can be subjected to visual inspection.
Such a dendrogram is a ternary tree where the leaves or external
nodes are
the basic data elements. Two leaves are connected to an internal node if
they are more similar to one another than to the other data elements.
Dendrograms are used in computational biology to illustrate the 
clustering of genes or the evolutionary tree of species. In the latter
case we want a rooted tree to see the order in which groups
of species split off from one another. 

In biology dendrograms (phylogenies) are ubiquitous, 
and methods to reconstruct a
rooted dendrogram from a matrix of pairwise distances abound. One of these
methods is quartet tree reconstruction as explained in
Section~\ref{sect.quartet}. Since the biologists assume there
is a single right tree (the data are ``tree-like'')
they also assume one quartet topology, of the
three possible ones of every quartet, is the correct one. Hence their
aim is to embed (Definition~\ref{def.consistent}) 
the largest number of correct quartet topologies 
in the target tree.

\subsection{Related Work}\label{sect.relwork}
The quartet tree method is described in Section~\ref{sect.quartet}.
A much-used heuristic called the Quartet Puzzling problem was proposed in
\cite{SvH96}. There, 
the quartet topologies are provided
with a probability value, and for each quartet the topology with
the highest probability is selected (randomly, if there are more than one)
as the maximum-likelihood optimal
topology.
The goal is to find a dendrogram
that embeds these optimal quartet
topologies. In the biological setting it is assumed that the observed genomic data
are the result of an evolution in time, and hence can be represented as the leaves
of an evolutionary tree. Once we obtain a proper probabilistic evolutionary
model to quantify the evolutionary relations between the data we can search 
for the true tree. 
In a quartet method one determines the most likely quartet topology
for every quartet
under the given assumptions, and then searches for a ternary tree (a dendrogram) that
contains as many of the most likely 
quartets as possible. By Lemma~\ref{lem.unique},
a dendrogram is uniquely
determined by the set of 
embedded quartet topologies that
it contains. These quartet topologies are said to be consistent with the tree
they are embedded in.  Thus, if all quartets are embedded in the tree
in their most likely topologies, then it is certain 
that this tree is the
optimal matching tree for the given quartet topologies input data.
In practice we often find that the set
of given quartet topologies
are inconsistent or incomplete.  Inconsistency makes it impossible to
match the entire input quartet topology set even for the optimal, best
matching tree.  Incompleteness threatens the uniqueness of the optimal
tree solution.  Quartet topology inference methods also suffer from
practical problems when applied to real world data.  In many
biological ecosystems there is reticulation that makes the relations
less tree-like and more network-like. The data can be corrupted and
the observational process pollutes and makes errors.

Thus, one has to settle for
embedding as many most likely quartet
topologies as possible, do error correction
on the quartet topologies, and so on.  
Hence in phylogeny, finding
the best tree according to an optimization criterion may not be
the same thing as inferring the tree underlying the data set (which we tend to
believe, but usually
cannot prove, to exist). For $n$ objects, there are
$(2n-5)!! = (2n-5) \times  (2n-3) \times \cdots \times 3$ unrooted dendrograms.
To find the optimal tree
turns out to be
NP--hard, see Section \ref{sect.hard}, and hence infeasible in general.
There are two main avenues that have
been taken:

(i) Incrementally grow the tree in random order by stepwise addition
of objects in the locally optimal way, repeat this for different
object orders, and add agreement values on the branches, like DNAML
\cite{Fe81}, or Quartet Puzzling \cite{SvH96}.
These methods are fast, but suffer from the usual bottom-up problem:
a wrong decision early on cannot be corrected later. Another
 possible problem is as follows.
Suppose we have just 32 items. With Quartet Puzzling we incrementally
construct a quartet tree from a randomly ordered list of
elements, where each next element is optimally connected
to the current tree comprising the previous elements. We repeat
this process for, say, 1000 permutations. Subsequently, we
look for percentage agreement of subtrees common to all such trees.
But the number of permutations is about $2^{160}$, so why would
the incrementally locally optimal trees derived from 1000 random permutations be
a representative sample from which we can conclude anything about
the globally optimal tree?

(ii) Approximate the global optimum monotonically or compute it,
using a geometric algorithm or
dynamic programming \cite{BCGOP98}, linear programming \cite{WDGG05},
or semi-definite programming \cite{SWR07}.
These latter methods, other methods, 
as well as methods related to the Minimum Quartet Consistency (MQC) problem
(Definition~\ref{def.mqc}),
cannot handle more than 15--30 objects \cite{WDGG05,LTM05,PBE04,BJKLW99,SWR07}
directly,
even while
using farms of desktops.
To handle more objects one needs to construct a supertree from the constituent
quartet trees for subsets of the original data sets, \cite{RMWW04}, as
in \cite{LTM05,PBE04}, incurring again the bottom-up problem of being
unable to correct earlier decisions.

\subsection{Present Work}

The Minimum Quartet Tree Cost (MQTC) problem is proposed in 
Section~\ref{sect.mqtc} (Definition~\ref{def.mqtc}). 
It is a quartet method for 
general hierarchical clustering of nontree-like data
in non-biological areas that is also applicable to
phylogeny construction in biology.
In contrast to the MQC problem,
it is used for general hierarchical clustering. It does not suppose
that for every quartet a single quartet topology is the correct one.
Instead, we aim at optimizing the summed quartet topology costs. 
If we determine the quartet topology costs from a measure of
distance, then
the data themselves are not required: 
all that is used is a distance matrix.
To solve it we present 
a computational heuristic that is 
a Monte Carlo method, as opposed to deterministic methods
like local search, Section~\ref{sect.mc}. Our method is based on
a fast randomized hill-climbing heuristic of a new global
optimization criterion. Improvements that dramatically decrease 
the running speed
are given in Section~\ref{sect.previous}.
  The algorithm does not address the problem of how
to obtain the quartet topology weights from sequence data
\cite{Ke98,LBCKKZ01,Li03}, but takes as input
the weights of all quartet topologies and executes the step
of how to reconstruct the hierarchical clustering from there.
Alternatively, we can start from the distance matrix and 
construct the quartet topology cost as the sum of the distances between
the siblings, dramatically speeding up the
heuristic as in Section~\ref{sect.previous}. 
Since the method globally optimizes the tree it does not suffer from the 
the disadvantage treated in Item (i) of Section~\ref{sect.relwork}.
The running time is much faster than that of the methods treated
in Item (ii) of Section~\ref{sect.relwork}. It can also handle much larger
trees of at least 300 objects. 

The algorithm presented produces a sequence of candidate trees with the objects
as leaves. Each such candidate tree is scored as to how well the tree
represents the information in the weighted
quartet topologies on a scale of 0 to 1. If a new candidate scores better
than the previous best candidate, the former becomes the new best candidate.
The globally optimal tree has the highest score, so the algorithm
monotonically approximates the global optimum. The algorithm terminates
on a given termination condition.

In contrast to the general case of bottom-up and top-down methods, the new 
quartet method can undo what was done before and eventually reaches 
a global optimum. It does not assume that the data are tree like (and hence
that there is a single ``right tree''), but simply hierarchically clusters
data items in every domain. The scalability is improved by the 
reduction of the running time from $\Omega(n^4)$ per generation 
in the original version (with
an implementation of at least $O(n^5)$) to
$O(n^3)$ per generation
 in the current optimised version in  Section~\ref{sect.previous}
and implemented in CompLearn \cite{Ci03} from version 1.1.3 onward.
(Here $n$ is the number of data items.) Recently,
in \cite{CDGKP08} several alternative approaches to the 
here-introduced solution heuristic
are proposed. Some of the newly introduced heuristics
perform better both in results and running times than 
our old implementation. However, even the best heuristic in \cite{CDGKP08}
appears to have a slower running time for natural data 
(with $n=32$ typically over
50\%)
than the current version of our algorithm (CompLearn version
1.1.3 or later.)

In Section~\ref{sect.ncd} we treat compression-based distances 
and previous experiments with the MQTC heuristic 
using the CompLearn software. In Section~\ref{sect.splitstree}
compare performance and running time of MQTC heuristic
in CompLearn versions 0.9.7 
and 1.1.3 (before and after the speedup
in Section~\ref{sect.previous}) with those of other modern methods.
These are UPGMA, BioNJ, and NJ,
as implemented in the SplitsTree version 4.6.
We consider artificial and natural data sets. 
Note that biological packages like SplitsTree assume
tree-like data and are not designed to deal with arbitrary hierarchical
clustering like the MQTC heuristic. The artificial and natural data sets  
we use are tree structured.
Thus, the comparison is unfair to the new MQTC heuristic.
 
\subsection{Origin and Computational Complexity}
The MQTC problem and heuristic
were originally  proposed in 
\cite{CVW03,CV04,CV07}.
There, the main focus is on compression-based distances, but
to visually present the results in tree form 
we focused on a quartet method
for tree reconstruction.
We also believed such a quartet tree 
method to be more sensitive and objective 
than other methods. The available
quartet tree methods were too slow
when they were exact or global, and too inaccurate or uncertain
when they were statistical incremental. They
also addressed only biological phylogeny. Hence, 
we developed a new approach
aimed at general hierarchical clustering.
This approach is not a top-down or bottom-up method 
that can be caught in
a local optimum. In the above references the approach is described
as an auxiliary notion in one or two pages. It is a 
major new method to do general hierarchical clustering.
Here we give the first complete treatment. 

Some details of the MQTC problem, its computational hardness,
and our heuristic for its solution, are as follows. 
The goal is to use a quartet
method to obtain high-quality
hierarchical clustering of data from arbitrary (possibly heterogeneous)
domains, not necessarily only biological phylogeny data.
Traditional quartet methods derive from biology.
There, one  assumes that there exists
a true evolutionary tree, and the aim is to embed as many
optimal quartet topologies as is possible. In the new method for
general hierarchical clustering, for $n$ objects
we consider all $3{n \choose 4}$ possible quartet topologies, each with
a given weight. Our goal is to find the tree such that the summed weights
of the embedded quartet topologies is optimal. We develop a 
randomized hill-climbing heuristic that
monotonically approximates this optimum, and a figure of merit
(Definition~\ref{def.st}) that quantifies
the quality of the best current candidate tree on a linear scale.
We give an explicit proof of
NP-hardness (Theorem~\ref{theo.NPhard}) of the MQTC problem.
Moreover, if a polynomial time approximation scheme (PTAS)
(Definition~\ref{def.ptas})  for the problem
exists, then
P=NP (Theorem~\ref{theo.ptas}).
Given the
NP--hardness of phylogeny reconstruction in general relative to
most commonly-used criteria, as well as the non-trivial
algorithmic and run-time complexity of all previously-proposed
quartet-based heuristics, such a simple heuristic is potentially
of great use.

\subsection{Materials and Scoring}
The data samples we used, here or in referred-to previous work,
 were obtained from standard data bases
accessible on the Internet, generated by ourselves,
or obtained from research groups in the field of investigation.
Contrary to biological phylogeny methods,
we do not have agreement  values on the branches:
we generate the best tree possible, globally balancing all requirements.
The quality of the results depends on
how well the hierarchical tree represents the information
in the set of weighted quartet topologies. 
The MQTC clustering heuristic generates a tree
together with a goodness score.
The latter is called standardized benefit score or $S(T)$ value
in the sequel (Definition~\ref{def.st}).
In certain natural data sets, such as H5N1 genomic sequences, consistently high
$S(T)$ values are returned even for large sets of objects of 100 or more nodes,
\cite{Ci07}.
In other nontree-structured natural data sets however, 
as treated in \cite{CVW03,CV04},
the $S(T)$ value deteriorates
more and more with increasing number of elements being put in the same tree.
The reason is that with increasing size of a nontree-structured
 natural data set
the projection of the information in the cost function into a
ternary tree may get increasingly distorted. This is because the underlying
structure in the data is incommensurate with any tree shape whatsoever.
In this way, larger structures may induce additional ``stress'' in the mapping
that is visible as lower and lower $S(T)$ scores.
Experience shows that in nontree-structured data 
the MQTC hierarchical clustering method
seems to work best for small sets of data, up to 25 items, and to deteriorate
for some (but by no means all) 
larger sets  of, say, 40 items or more.  This deterioration
is directly observable in the $S(T)$ scores 
and degrades solutions in two
common forms. The first form is tree instability when 
different or very different solutions are returned on 
successive runs. The second form is tree ``overlinearization''
when some data sets produce caterpillar-like structures 
only or predominantly.

In case a large set of objects, say 100 objects, clusters with high $S(T)$
value this is evidence that the data are of themselves tree-like, and
the quartet-topology weights, or underlying distances, truely represent to
similarity relationships between the data.
Generating trees from the same weighted quartet topologies
many times resulted in the same tree in case of high $S(T)$
value, or a similar tree in case of moderately high $S(T)$ value.
This happened for every weighting we used, even though the heuristic is randomized.
That is, there is only one way to be right, but increasingly
many ways to be increasingly wrong. 

\section{The Quartet Method}\label{sect.quartet}
Given a set $N$ of $n$ objects,
we consider every subset of four elements (objects) from our set
of $n$ elements;
there are ${n \choose 4}$ such sets. Such a set is called 
a {\em quartet}.
From each quartet $\{u,v,w,x\}$ we construct a tree of arity 3,
which implies that the tree consists of two subtrees of two
leaves each. Let us call such a tree a {\em quartet topology}.
We denote a partition $\{u,v\},\{w,x\}$ of $\{u,v,w,x\}$
by $uv | wx$.
There are
three possibilities to partition $\{u,v,w,x\}$ into
two subsets of two elements each: (i) $uv | wx$, (ii) $uw | vx$,
and (iii)  $ux | vw$.  In terms of the tree topologies:
 a vertical bar divides the two pairs of leaf nodes
into two disjoint subtrees (Figure~\ref{figquart}).

\begin{figure}[htb]
\begin{center}
\epsfig{file=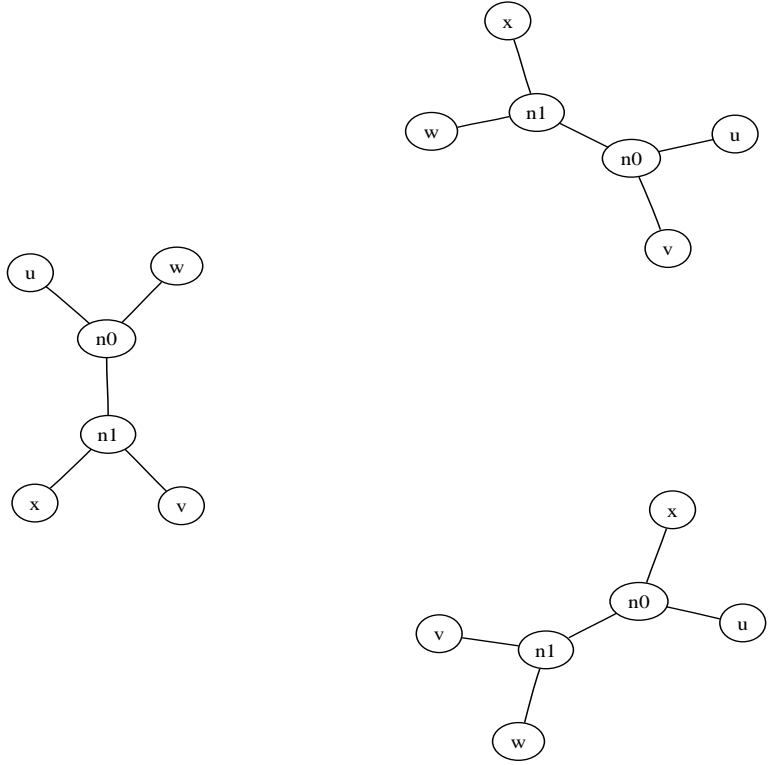,width=4in,height=4in}
\end{center}
\caption{The three possible quartet topologies for the set of leaf labels {\em u,v,w,x} }\label{figquart}
\end{figure}

The set of $3 {n \choose 4}$ quartet topologies induced by $N$
is denoted by $Q$.
Consider the class ${\cal T}$ of undirected trees
of arity 3 with $n \geq 4$ leaves (external nodes of degree 1),
labeled with the elements of $N$.
Such trees have $n$ leaves and $n-2$ internal nodes (of degree 3).
\begin{definition}\label{def.consistent}
\rm
For tree $T \in {\cal T}$ and 
four leaf labels $u,v,w,x \in N$, we say $T$ is {\em consistent}
with $uv | wx$, or the quartet topology $uv | wx$ is {\em embedded} in $T$,
if and only if the path from $u$ to $v$ does not cross
the path from $w$ to $x$. 
\end{definition}
It is easy to see that
precisely one of the three possible
quartet topologies of a quartet of four leaves is consistent
for a given tree from ${\cal T}$. Therefore a tree from ${\cal T}$
contains precisely ${n \choose 4}$ different quartet topologies.
\begin{figure}[htb]
\begin{center}
\epsfig{file=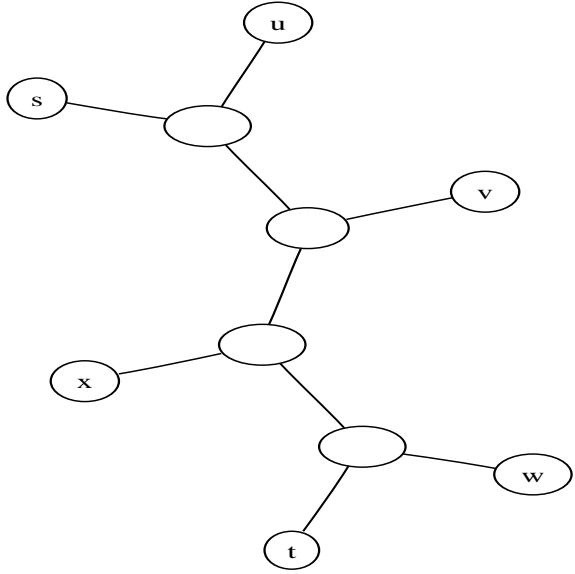,width=3in,height=3in}
\end{center}
\caption{An example tree consistent with quartet topology $uv | wx$ }\label{figexquart}
\end{figure}
Commonly the goal in the quartet method
is to find (or approximate as closely as possible) the tree
that embeds the maximal number of consistent (possibly weighted) quartet
topologies from a given set $P \subseteq Q$ of quartet topologies
\cite{Ji01} (Figure~\ref{figexquart}).
A  {\em weight function} $W: P \rightarrow {\cal R}$, with ${\cal R}$
the set of real numbers determines the weights. The unweighted case is
when $W(uv|wx)=1$ for all $uv|wx \in P$.
\begin{definition}\label{def.mqc}
\rm
The (weighted) {\em Maximum Quartet Consistency (MQC)
optimization} problem is defined as follows:

GIVEN: $N$, $P$, and $W$.

QUESTION: Find $T_0 = \max_{T \in {\cal T}} \{  \sum \{ W(uv|wx): uv|wx \in P$
and $uv|wx$ is consistent
with $T \}$.
\end{definition}

\section{Minimum Quartet Tree Cost}\label{sect.mqtc}
The rationale for the MQC optimization problem 
reflects the
genesis of the method in biological phylogeny. Under the
assumption that biological species developed by evolution in time, and $N$ is
a subset of the now existing species, there is a phylogeny 
$T_P \in
{\cal T}$ that represents that evolution. 
The set of quartet topologies
consistent with this tree has one quartet topology per quartet which is
the true one. The quartet topologies in $T_P$ are the ones which we assume
to be among the true quartet topologies, and weights are used to express
our relative certainty about this assumption concerning the individual
quartet topologies in $T_P$.

However, the data may be corrupted so that this assumption is no longer true.
In the general case of hierarchical clustering we do not even have a priori
knowledge that certain quartet topologies are objectively
true and must be embedded.
Rather, we are in the position that we can somehow assign a relative importance
to the different quartet topologies. Our task is then to balance the
importance of embedding different quartet topologies against one another,
leading to a tree that represents the concerns as well as possible.
We start
from a cost-assignment to the quartet topologies:
Given a set $N$ of $n$ objects, let $Q$ be the set of
quartet topologies, and let $C:Q \rightarrow {\cal R}$ be a {\em cost function}
assigning a real valued cost $C(uv|wx)$ to each quartet
$uv|wx \in Q$.
\begin{definition}\label{def.costs}
The {\em cost} $C_T$ of a tree $T$ with a set $N$ of leaves
is defined by
$C_T =\sum_{\{u,v,w,x\} \subseteq N} \{C(uv|wx): T$ is consistent
with  $uv |wx\}$---the
sum of the costs of all its consistent quartet topologies.
\end{definition}

\begin{definition}\label{def.mqtc}
\rm
Given $N$ and $C$,
the {\em Minimum Quartet Tree Cost (MQTC)} is
$\min_{T \in {\cal T}} \{C_T:$ $T$ is a tree with
the set $N$ labeling its leaves$\}$.
\end{definition}

We normalize the problem of finding the MQTC as follows:
Consider the list of all possible quartet topologies
for all four-tuples of labels
under consideration.  For each group of
three possible quartet topologies for a given
set of four labels $u,v,w,x$, calculate a best (minimal) cost
$m(u,v,w,x) = \min \{ C(uv|wx), C(uw|vx), C(ux|vw) \}$,
and a worst (maximal)
cost $M(u,v,w,x) = \max \{ C(uv|wx), C(uw|vx), C(ux|vw) \}$.
Summing all best quartet topologies yields the best (minimal) cost
$m = \sum_{\{u,v,w,x\} \subseteq N} m(u,v,w,x)$.
Conversely, summing all worst quartet topologies yields the worst (maximal) cost
$M =  \sum_{\{u,v,w,x\} \subseteq N} M(u,v,w,x)$.
For some cost functions,
these minimal and maximal values can not be attained by actual trees;
however, the score $C_T$ of every tree $T$ will lie between these two values.
In order to be able to compare the scores of quartet trees for different
numbers of objects in a uniform way,
we now rescale the score linearly such that the worst score maps to 0,
and the best score maps to 1:

\begin{definition}\label{def.st}
\rm
The {\em normalized tree benefit score} $S(T)$ is defined
by $S(T) = (M-C_T)/(M-m)$.
\end{definition}

Our goal is to find a full tree with a maximum value
of $S(T)$, which is to say, the lowest total cost.
Now we can rephrase the MQTC problem in such a way that
solutions of instances of different sizes can be
uniformly compared in terms of relative quality:

\begin{definition}
\rm
Definition of the {\em MQTC optimization problem}:

GIVEN: $N$ and $C$.

QUESTION: Find a tree $T_0$ with $S(T_0)=\max \{S(T):$ $T$ is a tree with
the set $N$ labeling its leaves$\}$.
\end{definition}

\begin{definition}
\rm
Definition of the {\em MQTC decision problem}:

GIVEN: $N$ and $C$ and a rational number $0 \leq k \leq 1$.

QUESTION: Is there a binary tree $T$ with
the set $N$ labeling its leaves and $S(T) \geq k$.
\end{definition}

\subsection{Computational Hardness}
\label{sect.hard}
The hardness of Quartet Puzzling is informally mentioned in the literature
\cite{WDGG05,LTM05,PBE04}, but we provide explicit proofs.
To express the notion of computational difficulty one uses
the notion of ``nondeterministic polynomial time (NP)''.
If a problem concerning $n$ objects is NP--hard
this means that the best known algorithm
for this (and a wide class of significant problems) requires
computation time at least exponential in $n$. That is, it is infeasible
in practice.
Let $N$ be a set of $n$ objects, let $T$ be a tree of which the $n$
leaves are labeled by the objects, and let $Q$ be the set of quartet topologies
and $Q_T$ be the set of quartet topologies
embedded in $T$.
\begin{definition}
\rm
The {\em MQC decision problem} is the following:

GIVEN: A set of quartet topologies $P \subseteq Q$,
and an integer $k$.

DECIDE: Is there
a binary tree $T$ such that $P \bigcap Q_T > k$.
\end{definition}
In \cite{St92} it is shown that the MQC decision problem
is NP--hard. Sometimes this problem is called the
 {\em incomplete} MQC decision
problem.  The less general {\em complete MQC decision problem} requires
$P$ to contain precisely one quartet topology per quartet (that is,
per each subset of $4$ elements out of the $n$ elements),
and is proved to be NP--hard as well in \cite{BJKLW99}.

\begin{theorem}\label{theo.NPhard}
(i) The MQTC decision problem is NP--hard.

(ii) The MQTC optimization problem is NP--hard.
\end{theorem}
\begin{proof}
(i) By reduction from the MQC decision problem.
   For every MQC decision problem one can define a corresponding
MQTC decision problem that has the same solution: give the
quartet topologies in $P$ cost 0 and the ones in $Q - P$ cost 1.
Consider the MQTC decision problem: is there a tree $T$ with the
set $N$ labeling its leaves such that $C_T < {n \choose 4} -k$ ?
An alternative equivalent formulation is: is there a tree $T$ with the
set $N$ labeling its leaves such that
\[
S(T) > \frac{M- {n \choose 4}+k}{M-m} \; \; ?
\]
Note that every tree $T$ with the
set $N$ labeling its leaves has precisely one out of the three
quartet topologies of every of the ${n \choose 4}$ quartets
embedded in it. Therefore, the cost $C_T = {n \choose 4}-|P \bigcap Q_T |$.
If the answer to the above question is affirmative,
then the number of quartet topologies in $P$ that
are embedded in the tree exceeds $k$; if it is not then there
is no tree such that the number of quartet topologies in $P$
embedded in it exceeds $k$.
This way the MQC decision problem can be
reduced to the MQTC decision problem, which shows also the latter
to be NP--hard.

(ii) An algorithm for the MQTC optimization problem yields
an algorithm for the MQTC decision problem with the same running
time up to a polynomial additive term: If the answer to
the MQTC optimization problem is a tree $T_0$, then we determine $S(T_0)$
in $O(n^4)$ time. Let $k$ be the bound of the MQTC decision
problem. If $S(T_0) \geq k$ then the answer to the decision problem
is ``yes,'' otherwise ``no.'' 
\end{proof}

The proof shows that negative complexity results for MQC carry
over to MQTC. 
\begin{definition}\label{def.ptas}
\rm
A {\em polynomial time approximation scheme (PTAS)} is
a polynomial time approximation algorithm for an
optimization problem with a performance guaranty.
It takes an instance of an optimization problem and a parameter $\epsilon>0$,
and produces a solution of an optimization problem that is optimal up
to an $\epsilon$ fraction.
\end{definition}
For example, for the MQC  optimization problem as defined above,
a PTAS would produce a tree embedding at least $(1-\epsilon)|P|$ quartets
from $P$.
The running time of a PTAS is required to be polynomial in the size of the
problem concerned for every fixed $\epsilon$, but can be different for
different $\epsilon$.
In \cite{BJKLW99} a
PTAS for a restricted version
of the MQC optimization problem, 
namely the ``complete'' MQC optimization problem
defined above, is exhibited.  This is
a theoretical approximation that
would run in something like $n^{19}$. For general (what we 
have called ``incomplete'')
MQC optimization it is
shown that even such a theoretical algorithm does not exist, unless
P=NP.
\begin{theorem}\label{theo.ptas}
If a PTAS for the MQTC optimization problem
exists, then P=NP.
\end{theorem}
\begin{proof}
The reduction in the proof of Theorem~\ref{theo.NPhard}
yields a restricted version of the MQTC optimization problem that is
equivalent to the MQC optimization problem. There is an isomorphism between
every partial solution, including the optimal solutions involved:
For every tree $T$ with $N$ labeling the leaves,
the MQTC cost $C_T = {n \choose 4}-|P \bigcap Q_T |$ where
$P \bigcap Q_T$ is the set of MQC consistent quartets.
The reduction is also
poly-time approximation preserving, since the
reduction gives a linear time computable
isomorphic version of the MQTC problem instance
for each MQC problem instance. 
Since \cite{BJKLW99} has shown that a
PTAS for the MQC optimization problem does not exist unless P=NP,
it also holds for this restricted version of the
MQTC optimization problem that a PTAS does not exist unless P=NP,
The full MQTC optimization problem is at least as hard to approximate
by a PTAS, from which the theorem follows.
\end{proof}

Is it possible that the best $S(T)$ value is always one, that is,
there always exists a tree that embeds all quartets at minimum
cost quartet topologies?
Consider the case $n=|N|=4$. Since there is only one quartet,
we can set $T_0$ equal to the minimum cost quartet topology,
and have $S(T_0)=1$.
A priori we cannot exclude the possibility that
for every $N$ and $C$ there always is a tree $T_0$ with $S(T_0)=1$.
In that case, the MQTC optimization problem reduces to finding that $T_0$.
However, the situation turns out to be more complex. Note first
that the set of quartet topologies uniquely determines
a tree in ${\cal T}$, \cite{Bu71}.

\begin{lemma}\label{lem.unique}
Let $T,T'$ be different labeled trees in ${\cal T}$ and let $Q_T,Q_{T'}$ be
the sets of embedded quartet topologies, respectively. Then,
$Q_T \neq Q_{T'}$.
\end{lemma}

A {\em complete set} of quartet topologies on $N$ is a set containing
precisely one quartet topology per quartet.
There are $3^{n \choose 4}$ such combinations, but only $2^{n \choose 2}$
labeled undirected graphs on $n$ nodes
(and therefore $|{\cal T}| \leq 2^{n \choose 2}$).
Hence, not every complete set of quartet topologies
corresponds to a tree in ${\cal T}$. This already suggests that we can
weight the quartet topologies in such a way that the full combination
of all quartet topologies at minimal costs does not correspond to
a tree in ${\cal T}$, and hence $S(T_0) < 1$ for $T_0 \in {\cal T}$
realizing the MQTC optimum. For an explicit example of this, we use
that a complete set corresponding to a tree in ${\cal T}$
 must satisfy certain transitivity properties,
 \cite{CS77,CS81}:

\begin{lemma}\label{cl1}
Let $T$ be a tree in the considered class with leaves $N$, $Q$ the set
of quartet topologies and
$Q_0 \subseteq Q$. Then
$Q_0$ uniquely determines $T$ if

(i) $Q_0$ contains precisely one quartet topology for every quartet,
and

(ii) For all $\{a,b,c,d,e\} \subseteq N$,
if $ab|bc, ab|de \in Q$ then $ab|ce \in Q$, as well as
if $ab|cd, bc|de \in Q$ then $ab|de \in Q$.
\end{lemma}

\begin{theorem}
There are $N$ (with $n=|N|=5$) and a cost function $C$ such that,
for every $T \in {\cal T}$,
$S(T)$ does not exceed $4/5$.
\end{theorem}
\begin{proof}
Consider $N=\{u,v,w,x,y\}$ and $C(uv|wx) = 1-\epsilon (\epsilon > 0),
C(uw|xv)= C(ux|vw)=0$,
$C(xy|uv)=C(wy|uv)=C(uy|wx)=C(vy|wx)=0$, and $C(ab|cd)=1$ for
all remaining quartet topologies $ab|cd \in Q$.
We see that $M= 5 - \epsilon$, $m=0$.
The tree $T_0 = (y,((u,v),(w,x)))$ has cost $C_{T_0}= 1-\epsilon$,
since it embeds quartet topologies $uw|xv, xy|uv, wy|uv, uy|wx, vy|wx$.
We show that $T_0$ achieves the MQTC optimum.

{\em Case 1:}
If a tree $T \neq T_0$ embeds $uv|wx$, then it
must by Item (i) of Lemma~\ref{cl1}
also embed a quartet topology
containing $y$ that has cost 1.

{\em Case 2:}
If a tree $T \neq T_0$ embeds $uw|xv$ and $xy|uv$, then it must by
Item (ii) of the Lemma~\ref{cl1}
also embed $uw|xy$,
and hence have cost $C_T \geq 1$. Similarly, all other
remaining cases of embedding a combination of a quartet
topology not containing $y$ of 0 cost with a quartet topology containing
$y$ of 0 cost in $T$, imply an embedded
quartet topology of cost 1 in $T$.
\end{proof}

Altogether,
the MQTC optimization problem is infeasible in practice, and natural data
can have an optimal $S(T)< 1$. In fact, it follows from the above
analysis that to determine the optimal $S(T)$ in general is NP--hard.
If the deterministic
approximation of this optimum to within a given precision
can be done in polynomial time, then
that implies the generally disbelieved conjecture P=NP.
Therefore, any practical approach to obtain or approximate the
MQTC optimum requires some type of heuristics, for example
Monte Carlo methods.

\section{Monte Carlo Heuristic}\label{sect.mc}
Our algorithm is a Monte Carlo heuristic,
essentially randomized hill-climbing 
where
undirected trees evolve in a random walk
driven by a prescribed fitness function. We are given a set $N$ of
$n$ objects and a cost function $C$.

\begin{definition}
\rm
We define a {\em simple mutation} on a labeled undirected ternary tree
as one of the following possible transformations:
\begin{enumerate}
\item A {\em leaf interchange}: 
randomly choose two leaves that are
not siblings and interchange them.
\item A {\em subtree interchange}: randomly choose two 
internal nodes $u,w$, or an internal node $u$ and a leaf $w$,
such that the shortest path length between $u$ and $w$ is at
least three steps. That is, $u-x- \cdots -y-w$   
is a shortest path in the tree. 
Disconnect $u$ (and the subtree rooted at $u$
disjoint from the path) from $x$, and disconnect $w$ 
(and the subtree rooted at $w$ 
disjoint from the path if $w$ is not a leaf) 
from $y$. Attach $u$ and its subtree 
to $y$, and $w$ (and its subtree if $w$ is not a leaf) to $x$.
\item A {\em subtree transfer}, whereby a randomly chosen subtree 
(possibly a leaf) is detached and reattached in another place, 
maintaining arity invariants.
\end{enumerate}
\end{definition}
Each of these simple mutations keeps the
number of leaf nodes and internal nodes in the tree invariant;
only the structure and placements
change. Clearly, mutations 1) and 2) can be together replaced by the single 
mutation below. But in the implementation they are separated as above.
\begin{itemize}
\item A {\em subtree and/or leaf interchange}, which consists of randomly choosing two
nodes (either node or both can be leaves or internal nodes), say $u,w$,
such that the shortest path length between $u$ and $w$ is at
least three steps. That is, $u-x- \cdots -y-w$
is a shortest path in the tree. Disconnect $u$ (and the subtree rooted
at $u$
disjoint from the path) from $x$, and disconnect $w$
(and the subtree rooted at $w$
disjoint from the path) from $y$. Attach $u$ and its subtree
to $y$, and $w$ and its subtree to $x$.
\end{itemize}
A sequence of these mutations suffices to go from every ternary 
tree with $n$ labeled leaves
and $n-2$ unlabeled internal nodes to any other ternary tree 
with $n$ labeled leaves
and $n-2$ unlabeled internal nodes, Theorem~\ref{theo.mutations}
in Appendix~\ref{sect.mutations}.
\begin{definition}
\rm
A {\em $k$-mutation} is a sequence of $k$
simple mutations.
Thus, a simple mutation is a 1-mutation.
\end{definition}

\subsection{Algorithm}
The algorithm is given in Figure~\ref{fig.alg}. We comment on the different
steps:

{\em Comment on Step 2:} A tree is consistent with precisely
$\frac{1}{3}$ of all quartet topologies, one for every quartet.
A random tree is likely to be consistent with about $\frac{1}{3}$ of the best
quartet topologies---but because of dependencies this figure is
not precise.
\begin{figure}
\begin{center}
\begin{description}
\item{\bf Step 1:} First, a random tree $T \in {\cal T}$ with $2n-2$ nodes
is created, consisting of $n$ leaf nodes (with 1 connecting edge) labeled
with the names of the data items, and $n-2$ non-leaf or {\em internal} nodes.
When we need to refer to specific internal nodes, we
label them with the lowercase letter ``k'' 
 followed by a
unique integer identifier.  Each internal node has exactly
three connecting edges.

\item{\bf Step 2:} For this
tree $T$, we calculate the summed total cost of all embedded quartet topologies,
and compute  $S(T)$.

\item{\bf Step 3:}
The {\em currently best known tree} variable $T_0$ is set
to $T$:   $T_0 \leftarrow T$.

\item{\bf Step 4:}
Draw a number $k$ with
probability $p(1)=1-c$ and $p(k) = c /(k (\log k)^2)$ for $k \geq 2$,
where
$1/c = \sum_{k=2}^{\infty} 1/(k (\log k)^2)$.
By \cite{Br92} it is known that $1/c \approx 2.1$.

\item{\bf Step 5:}
Compose a $k$-mutation by,
for each of the constituent sequence of $k$ simple mutations, choosing one of
the three types listed above with equal probability.  For each of
these simple mutations, we uniformly at random select
leaves or internal nodes, as appropriate.

\item{\bf Step 6:}
In order to search for a better tree,
we simply
apply the $k$-mutation constructed in {\em Step 5}
to $T_0$ to obtain $T$, and then
calculate $S(T)$.  If $S(T) > S(T_0)$, then replace the current
candidate in $T_0$ by $T$ (as the new best tree):
$T_0 \leftarrow T$.

\item{\bf Step 7:}
If $S(T_0) =1$ or a {\em termination condition} to be discussed below holds,
then output the tree in $T_0$ as the best tree and halt.
Otherwise, go to {\em Step 4}.
\end{description}
\caption{The Algorithm}\label{fig.alg}
\end{center}
\end{figure}

{\em Comment on Step 3:} This $T_0$  is used as
the basis for further searching.

{\em Comment on Step 4:}
This number $k$ is the number of simple mutations that we will constitute
the next $k$-mutation. 
The probability mass function $p(k)$ for $k\geq 2$ is 
$p(k)=c/(k \log^2 k)$ with $c \approx 2.1$.
In practice, we used a ``shifted'' fat tail
probability mass function $1/((k+2) (\log k+2)^2)$ for $k \geq 1$.

{\em Comment on Step 5:} Notice
that trees which are close to the original tree (in terms of number of
simple mutation steps in between) are examined often, while trees that are
far away from the original tree will eventually be examined, but not very
frequently.

\begin{remark}
\rm
We have chosen $p(k)$ to be a ``fat-tail'' distribution,
with one of the fattest tails possible,
to concentrate maximal probability also on the larger values of $k$.
That way, the likelihood of getting trapped in local minima is minimized.
In contrast, if one would choose an exponential scheme, like
$q(k)=c e^{-k}$, then the larger values of $k$ would arise so scarcely
that practically speaking the distinction between being absolutely trapped in
a local optimum, and the very low escape probability, would be
insignificant. Considering positive-valued probability mass
functions $q: {\cal N} \rightarrow (0,1]$, with ${\cal N}$
the natural numbers, as we do here, we note that such a function
(i) $\lim_{k \rightarrow \infty} q(k) =0$,
 and (ii) $\sum_{k=1}^{\infty} q(k) =1$.
Thus, every function of the natural numbers
that has strictly positive values and converges can be normalized to such
a probability mass function. For smooth analytic functions that can be expressed
as a series of fractional powers and logarithms, the borderline between
converging and diverging is as follows: $\sum 1/k, \sum 1/(k \log k)$,
$\sum 1/(k \log k \log \log k)$ and so on diverge, while
$\sum 1/k^2, \sum 1/(k (\log k)^2)$,$\sum 1/(k \log k (\log \log k)^2)$
 and so on
converge. Therefore,
the maximal fat tail of a ``smooth'' function $f(x)$
with $\sum f(x) < \infty$  arises for functions at the edge of the convergence
family. The probability mass function $p(k)= c /(k (\log k)^2)$ is as close to the edge
as is reasonable.
Let us see what this
means for our algorithm using the chosen probability mass function $p(k)$
where we take $c = \frac{1}{2}$ for convenience.

For $n=32$ we can change any tree in ${\cal T}$ 
to any other tree in
${\cal T}$ with a squence of at most $5n-16 = 144$ simple 
mutations (Theorem~\ref{theo.mutations} 
in Appendix~\ref{sect.mutations}).
The probability of such a complex mutation
occurring is quite large with such a fat tail: 
$\approx 1/(2 \cdot 144 \cdot 7^2) = 
1/14112$. The expectation is
about 7 times in 100,000 generations. 
The $5n-16$ is a crude upper bound; we
believe that the real value is more likely to be about $n$ simple
mutations. The probability of a sequence of $n$ simple mutations
occurring is $\approx 1/(2 \cdot 32 \cdot 5^2) = 1/1600$. The
expectation increases to about 63 times in 100.000 generations.
If we can already get out of a local
minimum with only a 16-mutation, then this occurs with probability
is $1/512$, so it is expected about 195 times in 
100.000 generations, 
and with an 8-mutation the probability
is $1/144$, so the expectation is more than
694 times in 100.000 generations.
\end{remark}

\subsection{Performance}
The main problem with hill-climbing algorithms is that they can get stuck
in a local optimum. However, by randomly selecting a sequence of simple
mutations, longer sequences with decreasing probability, we essentially
run a 
of simulated annealing \cite{KGV83} algorithm at random temperatures.
Since there is a nonzero probability for every tree in ${\cal T}$
being transformed into every other tree in ${\cal T}$, there is zero
probability that we get trapped forever
in a local optimum that is not a global optimum.
That is, trivially:
\begin{lemma}
(i) The algorithm approximates the MQTC optimal solution monotonically
in each run.

(ii) The algorithm without termination condition
solves the MQTC optimization problem eventually with probability 1
(but we do not in general know when the optimum has been
reached in a particular run).
\end{lemma}

\begin{figure}[htb]
\begin{center}
\epsfig{file=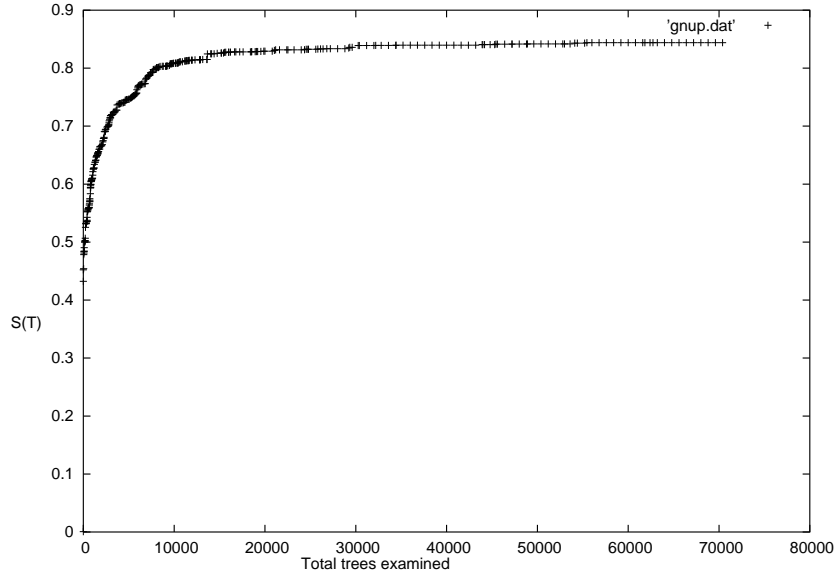,width=3in,angle=270}
\end{center}
\caption{Progress of a 60-item data set experiment over time}\label{figprogress}
\end{figure}

The main question therefore is the convergence speed of the algorithm
on natural data in terms of $S(T)$ value, 
and a termination criterion to terminate the
algorithm when we have an acceptable approximation. Practically,
from Figure~\ref{figprogress} it appears that improvement in terms of $S(T)$
eventually gets less and less (while improvement is still possible) in terms of
the expended computation time. Theoretically, this is explained
by
Theorem~\ref{theo.ptas} which tells us that there is no
polynomial-time approximation scheme for MQTC optimization.
Whether our approximation scheme is expected 
polynomial time seems to require
proving that the involved Metropolis chain is rapidly mixing
\cite{Vi00}, a notoriously hard and generally unsolved problem.
However, in our experiments there is
unanimous evidence that for the natural data and the cost
function we have used, convergence to close to the optimal
value is always fast.

The running time is determined as follows. We have to
determine the cost of ${n \choose 4}$ quartets to determine
each $S(T)$ value in each generation. Hence, trivially,
\begin{lemma}
The algorithm in Figure~\ref{fig.alg} runs in time $\Omega (n^4)$
per generation where $n$ is the number of objects.
(The implementation uses even $\Omega (n^5)$ time.)
\end{lemma}
\begin{remark}
The input to the algorithm in Figure~\ref{fig.alg} is the
quartet topology costs.
If one constructs the quartet-topology costs from more basic quantities,
such as the cost of $ab|cd$ equals the sum of the distances $d(a,b)+d(c,d)$ for
some distance measure $d(\cdot, \cdot)$, then one can use the additional
structure thus supplied to speed up the algorithm as 
in Section~\ref{sect.previous}. Then,
while the original implementation
of the algorithm uses as much as $\Omega (n^5)$ time per
generation it is sped up to $O(n^3)$ time per generation, 
Lemma~\ref{lem.previous}, and we were able to analyze 
a 260-node tree in about 3 hours cpu time reaching $S(T) \approx 0.98$.
\end{remark}
In experiments we found that for the same data set different
runs consistently showed the same behavior, for example
Figure~\ref{figprogress} for a 60-object computation. There
the $S(T)$ value leveled off at about 70,000 examined trees,
and the termination condition was ``no improvement in 5,000 trees.''
Different random runs of the algorithm nearly always gave the same behavior,
returning a tree with the same $S(T)$ value, albeit a different tree in
most cases since here $S(T) \approx 0.865$, a relatively low value.
That is, there are many ways to find a tree of optimal
$S(T)$ value if it is low, and apparently the algorithm never got trapped
in a lower local optimum. For problems with high $S(T)$ value
the algorithm consistently returned the same
tree. 

Note that if a tree is ever found such that $S(T) = 1$, then we can stop
because we can be certain that this tree is optimal, as no tree could
have a lower cost.  In fact, this perfect tree result is achieved in our
artificial tree reconstruction experiment (Section~\ref{sect.artificial})
reliably for 32-node trees in a few seconds using the improvement
of Section~\ref{sect.previous}.
For real-world data, $S(T)$ reaches a maximum somewhat
less than $1$. This presumably reflects distortion of the information
in the cost function
data by the best possible tree representation, 
or indicates getting stuck in a local optimum. 
Alternatively, the search space is too large
to find the global optimum.

On typical problems of up to 40 objects this tree-search gives a tree
with $S(T) \geq 0.9$ within half an hour (fot the unimproved version) and
a few seconds with the improvement of Section~\ref{sect.previous}. 
 For large numbers of objects,
tree scoring itself can be slow especially without the improvements
of Section~\ref{sect.previous}.
Current single computers can score a tree of this size with the 
unimproved algorithm in about a minute.
For larger experiments, we used the C program called
partree (part of the CompLearn package \cite{Ci03})
with MPI (Message Passing Interface, a common
standard used on massively parallel computers) on a cluster of workstations in
parallel to find trees more rapidly.  We can consider the graph
mapping the achieved $S(T)$ score as a function
of the number of trees examined.  Progress
occurs typically in a sigmoidal fashion towards a maximal value $\leq 1$,
Figure~\ref{figprogress}.

\subsection{Termination Condition}
The {\em termination condition} is of two types and which type
is used determines the number of objects we can handle.

{\em Simple termination condition:} We simply run the
algorithm until it seems
no better trees are being found in a reasonable amount of time.
Here we typically terminate if no improvement in $S(T)$ value is
achieved within 100,000 examined trees. This criterion is simple
enough to enable us to hierarchically cluster data sets up to 80
objects in a few hours even without the improvement
in Section~\ref{sect.previous} and at least up to 300
objects with the improvement. This is way above the 15--30 objects in
the previous exact (non-incremental) methods (see Introduction).

{\em Agreement termination condition:} In this more sophisticated method we
select a number $2 \leq r \leq 6$ of runs, and we run $r$ invocations
of the algorithm in parallel. Each time an $S(T)$ value in run $i=1, \ldots, r$
is increased in this process it is compared with the $S(T)$ values
in all the other runs. If they are all equal, then the candidate trees
of the runs are compared. This can be done by simply comparing the
ordered lists of embedded quartet topologies, in some standard order.
This works
since the  set of embedded quartet topologies uniquely
determines the quartet tree by \cite{Bu71}. If the $r$ candidate trees
are identical, then terminate with this quartet tree as output, otherwise
continue the algorithm.

This termination condition takes (for the same number of steps per run)
about $r$ times as long as the simple termination condition.
But the termination
condition is much more rigorous, provided we choose $r$
appropriate to the number $n$
of objects being clustered.
 Since all the runs are randomized independently at startup,
it seems very unlikely that with natural data
all of them get stuck in the same local
optimum with the same quartet tree instance,
provided the number $n$ of objects being clustered is not
too small. For $n = 5$ and the number of invocations $r=2$,
there is a reasonable probability that the two different
runs by chance hit the same tree in the same step. This phenomenon
leads us to require more than two successive runs with exact agreement before
we may reach a final answer for small $n$.  In the case of $4\le n \le 5$, we
require 6 dovetailed runs to agree precisely before termination.  For $6 \le n
\le 9$, $r = 5$. For $10 \le n \le 15$, $r = 4$.  For $16 \le n \le 17$,
$r = 3$.  For all other $n \ge 18$, $r = 2$.  This yields a reasonable tradeoff
between speed and accuracy. These specifications of $r$-values
relative to $n$ are partially common sense, partially empirically derived.

It is clear that there is only one tree with $S(T)=1$ (if that is possible for
the data), and it is straightforward
that random trees (the majority of all possible quartet trees) have
$S(T) \approx 1/3$.  This gives evidence that the number of quartet
trees with large $S(T)$ values is much smaller than the number of trees with
small $S(T)$ values.  It is furthermore evident that the precise relation
depends on the data set involved, and hence cannot be expressed by a general
formula without further assumptions on the data. However, we can
safely state that small data sets, of say $\leq 15$ objects,
that in our experience often lead to $S(T)$ values close to 1 and
a single resulting tree have
very few quartet trees realizing the optimal $S(T)$ value. On the other
hand, those large sets of 60 or more objects that contain some inconsistency
and thus lead to a low final $S(T)$ value also tend to 
exhibit more variation in the resulting trees.
This suggests that in the agreement
termination method each run will get stuck in a different quartet
tree of a similar $S(T)$ value, so termination with the same tree
is not possible. Experiments show that with the rigorous agreement termination
we can handle sets of up to 40 objects, 
and with the simple termination
up to at least 80--200 objects on a single 
computer with varying degrees of quality and consistency depending on the data
involved,
even without the 
improvements of Section~\ref{sect.previous}.
Basically the algorithm evaluates
all quartet topologies in each generated tree, which leads to an $\Omega(n^4)$
algorithm per generation or $O(n^3)$ per generation for the improved
version in Section~\ref{sect.previous}. 
With the improvement one can 
attack problems of over 300 objects. 

Recently, \cite{CDGKP08} has
used various other heuristics different from the ones presented here
to obtain methods that are both faster and yield better results than our old
implementation. But even the best heuristic there appears 
to have a slower running time for natural data
(with $n=32$ typically over 50\%) than
our current implementation (in CompLearn) using the speedups
of Section~\ref{sect.previous}.

\subsection{Tree Building Statistics}
We used the CompLearn package, \cite{Ci03},
to analyze a ``10-mammals'' example with {\em zlib} compression
yielding a  $10 \times 10$ distance matrix, similar to the examples in
Section~\ref{sect.nat}.
The algorithm starts with four randomly initialized trees.
It tries to improve each one randomly and finishes when they match.
Thus, every run produces an output tree, a maximum score associated
with this tree, and has examined some total number of trees,
$T$, before it finished.
Figure~\ref{fig.plot}
shows a graph displaying a histogram
of $T$ over one thousand runs of the distance matrix.  The $x$-axis
represents a number of trees examined in a single run of the program, measured
in thousands of trees and binned in 1000-wide histogram bars.  The maximum
number is about 12000 trees examined.  \begin{figure}[htb]
\begin{center}
\epsfig{file=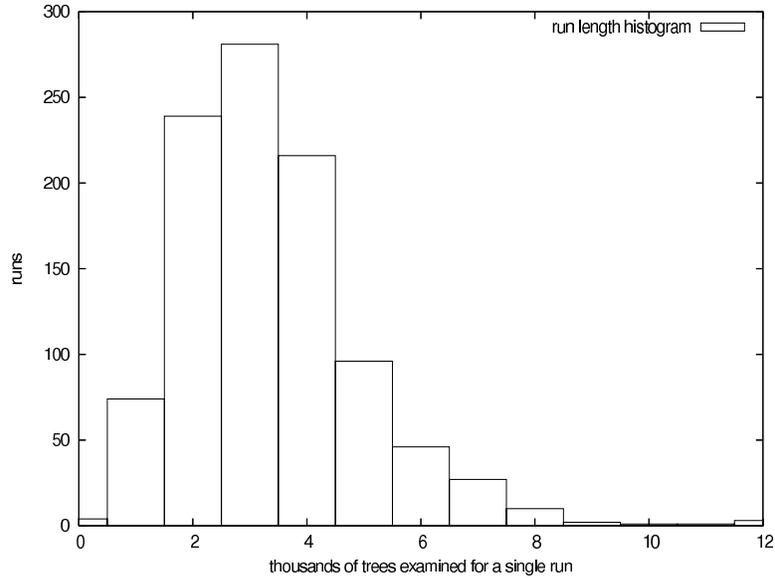,width=4in,height=3in}
\end{center}
\caption{Histogram of run-time number of trees examined before termination.}\label{fig.plot}
\end{figure}
The graph suggests a Poisson probability mass function.
 About $2/3$rd of the trials take less than 4000 trees.  In the thousand trials
above, 994 ended with the optimal $ S(T) = 0.999514 $.  The remaining six runs
returned 5 cases of the second-highest score, $  S(T) = 0.995198 $ and one case
of $ S(T) = 0.992222 $.  It is important to realize that outcome stability is
dependent on input matrix particulars.

Another interesting probability mass function is the mutation stepsize.
Recall that the mutation length is drawn from a shifted fat-tail probability mass function.
But if we restrict our attention to just the mutations
that improve the $S(T)$ value, then we may examine these statistics
to look for evidence of a modification to this distribution due to,
for example, the presence of very many isolated areas that have only
long-distance ways to escape.  Figure~\ref{fig.mutplot}
 shows the histogram
\begin{figure}[htb]
\begin{center}
\epsfig{file=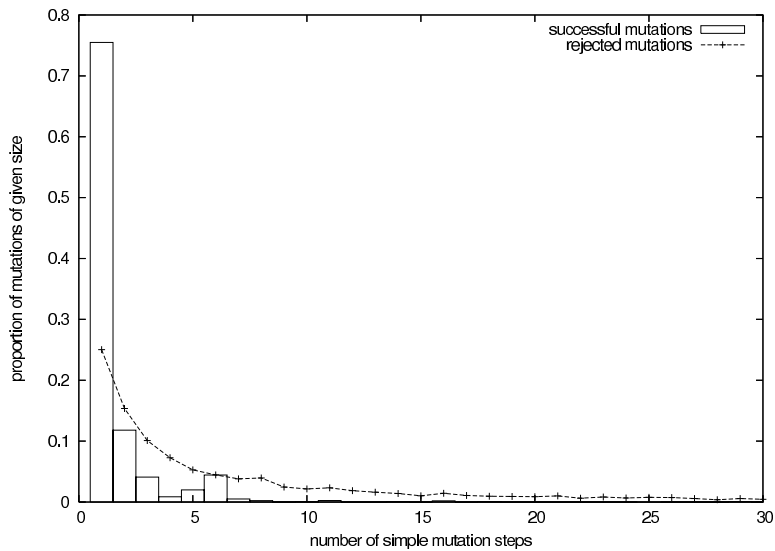,width=4in}
\end{center}
\caption{Histogram comparing probability mass functions of $k$-mutations per run.}\label{fig.mutplot}
\end{figure}
of successful mutation lengths (that is, number of simple mutations
composing a single “kept” complex mutation) and rejected lengths (both
normalized) which shows that this is not the case.
Here the $x$-axis is the number of mutation steps and the $y$-axis is
the normalized proportion of times that step size occurred.
This gives good empirical evidence that in this case,
at least, we have a relatively easy search space, without large gaps.

\subsection{Controlled Experiments}

With natural data sets, say genomic data, one may have the preconception
(or prejudice) that primates should be clustered together,
rodents should be clustered together, and so should ferungulates.
However, the genome of a marsupial
may resemble the genome of a rodent more than that of a
monotreme, or vice versa---the very question one wants to resolve.
Thus, natural
data sets may have ambiguous, conflicting, or counter intuitive
outcomes. In other words, the experiments on natural data sets have
the drawback of not having an objective clear ``correct'' answer that can
function as a benchmark for assessing our experimental outcomes,
but only intuitive or traditional preconceptions.
In Section~\ref{sect.artificial} the
experiments show that our
program indeed does what it is supposed to do---at least in
these artificial situations where we know in advance what 
the correct answer is.

\section{Improved Running Time}\label{sect.previous}
Recall that our quartet heuristic  
consists of two parts: (i) extracting the quartet topology costs
from the data, and (ii) repeatedly randomly mutating the current quartet tree
and determining the cost of the new tree. Both the MQTC problem
and the original heuristic is actually
concerned only with item (ii). To speed up the method we also
look at how the quartet topology costs can be derived from the distance matrix, that is,
we also look at item (i). 

{\bf Speedup by faster quartet topology cost computation:}
Assume that
the cost of a quartet topology is the sum
of the distances between each pair of neighbors
\begin{equation}\label{eq.costs}
C(uv|wx) = d(u,v) + d(w,x),
\end{equation}
for a given distance measure $d$. In the original implementation
of the heuristic one used the quartet topology costs to
calculate the $S(T)$ value of a tree $T$ without worrying
how those costs arose. Then, the heuristic runs in time order $n^4$
per generation
(actually $n^5$ if one counts certain details proportional 
to the internal path
length
as they were implemented).
If the quartet topology costs are derived according to
\eqref{eq.costs}, then we can reduce the running
time of the implementation by two orders of magnitude.
 
{\bf Subroutine  with
distance matrix input:} 

We compute the quartet topology costs according to \eqref{eq.costs}.
Let there be $n$ leaves in a ternary tree $T$. 
For every internal node (there are
$n-2$) determine a sum as follows. Let $I$ be the set of internal nodes
(the nodes in $T$ that are not leaves). There are three edges incident
with the internal node $p \in I$, say $e_1$, $e_2$ and $e_3$. Let the subtrees
attached to $e_i$ have $n_i$ leaves ($i=1,2,3$) so that 
$n=n_1+n_2+n_3$. For every edge $e_i$ there are ${{n_i} \choose 2}$
pairs of leaves in its subtree $T_i$. Each such pair can form a quartet
with pairs $(u,v)$ sich that $u$ is a leaf in the subtree $T_j$ of $e_j$
and $v$ is a leaf in the subtree $T_k$ of $e_k$. Define
$C(e_i)={{n_i} \choose 2} \sum_{u \in T_j, v \in T_k} d(u,v)$
 ($1 \leq i,j,k \leq 3$
and $i,j,k$ are not equal to one another). 
For internal node $p$ let the cost $C(p)=C(e_1)+C(e_2)+C(e_3)$. 
Compute the cost
$C_T$ of the tree $T$ as $C_T = \sum_{p \in I} C(p)$.

\begin{lemma} The cost
$C_T$ of the tree $T$ satisfies $C_T = \sum_{p \in I} C(p)$. If the tree $T$
has $n$ leaves, then the subroutine above runs in time
$O(n^2)$ per internal node and
hence in $O(n^3)$ overall to determine the sum $C_T$.
\end{lemma}
\begin{proof}
An internal 
node $p$ has three
incident edges, say $e_1,e_2,e_3$. 
Let $T_i$ be the subtree rooted at $p$ containing edge $e_i$
having $n_i$ leaves ($1 \leq i \leq 3$) so that $n=n_1+n_2+n_3$.
Let $(u,v)$ be a leaf pair with $u\in T_i$ and $v_\in T_j$,
and $w \neq x$ be leaves in $T_k$ ($1 \leq i,j,k \leq 3$
and $i,j,k$ unequal). Let $Q(p)$ be defined as
the set of all these 
quartet topologies $uv|wx$. Clearly,
if $p,q \in I$ and $p \neq q$, then $Q(p) \bigcap Q(q) = \emptyset$.
The $(u,v)$ parts of quartet topologies $uv|wx \in Q(p)$ 
have a summed cost $C(e_i)$ given by
$C(e_i)={{n_i} \choose 2} \sum_{u \in T_j, v \in T_k} d(u,v)$.
Then, the cost of $Q(p)$ is $C(p) = C(e_1)+C(e_2)+C(e_3)$. 

Every quartet topology $uv|wx$ in tree $T$ is composed of two pairs
$(u,v)$ and $(w,x)$. Every pair $(u,v)$ determines an internal node
$p$ such that the paths from $u$ to $p$ and from $v$ to $p$
are node disjoint except for the end internal node $p$. 
Every internal node $p$ determines a set of such pairs $(u,v)$,
and for different internal nodes the associated sets are disjoint.
All quartet topologies embedded in the tree $T$ occur this way.
Hence, 
\begin{equation}\label{eq.qt}
\bigcup_{p \in I} Q(p)=Q_T, 
\end{equation}
where $Q_T$ was earlier defined as the set of quartet topologies
embedded in $T$.  
By \eqref{eq.qt}, the summed cost $C_T$ 
(Definition ~\ref{def.costs}) of all quartet topologies embedded in $T$ 
satisfies $C_T= \sum_{p \in I} C(p)$.

The running time of determining 
$C(p)$
for a $p \in I$ is dominated by the summing of the $d(u,v)$'s 
for the pairs
$(u,v)$ of leaves in different subtrees with $p$ as root.
There are $O({n \choose 2})=O(n^2)$ such pairs.
Since $|I|=n-2$ the lemma follows.
\end{proof}

This immediately yields the following:
\begin{lemma}\label{lem.previous}
With as input a distance matrix between $n$ objects,
the quartet topology costs as in \eqref{eq.costs}, 
the subroutine above lets
the algorithm in Figure~\ref{fig.alg} (and its implementation)
run in time $O (n^3)$
per generation.
\end{lemma}

{\bf Speedup by MMC:} A Metropolis Markov Chain (MMC) \cite{MRRTT53}
 is implemented
inside the mutation chains of the algorithm. So,
instead of doing a qroup of mutations and at the
end check if the result improves upon the
original (that is, hill climbing), we do the following. After every mutation, a
Metropolis acceptance step is performed, rolling
back the changes when the step is rejected.
Acceptance is calculated on the raw scores of the
tree (unnormalized, thus being more selective
with larger trees). During the Metropolis walk,
the best tree found is kept, at the end this best
tree is returned and checked for improvement
(hill climbing).
This serves three purposes:
\begin{itemize}
\item
The search is faster because after every change, the
trees are focused on improving the $S(T)$ value.
This gives less spurious drift.
\item
The global search behavior is maintained, as there
is a nonzero probability that a tree is
transformed into any other tree.
\item
There is less dependency on the number of
mutations to perform in an individual step. It is
no longer necessary to try a few mutations
more often than many mutations, simply because
the trees are no longer allowed to drift away
very far from the current best in an unchecked
manner.
\end{itemize}
By Theorem~\ref{theo.mutations} in 
Appendix~\ref{sect.mutations}, every tree with $n-2$ unlabeled
internal nodes and $n$ labeled 
leaves can be transformed in every other
such tree in at most $5n-16$ simple mutations ($n \geq 4$).
We believe that the real value of the number of required
simple mutations is
about $n$, and therefore have set the trial length
to $n$.  The setting  does influence the
global search properties of the algorithm, longer
trial length meaning larger probability of
finding the global optimum. In the limit of
infinite trial length, the algorithm will behave
as a regular MMC algorithm with associated
convergence properties.

The newest version of the MQTC heuristic
is at \cite{Ke} and has been incorporated in
CompLearn \cite{Ci03} from version 1.1.3 onwards. 
Altogether, with both
types of speedup, the resulting speedup is
at least of the order of 1.000 to 10.000 for common sets of
objects with, say, $n \leq 300$. 

\section{Compression-based Distance}\label{sect.ncd}
To be able to make unbiased comparisons between phylogeny
reconstruction algorithms that take distance matrices as input, 
we use the
compression-based \NCD distance.
This metric distance
was co-developed by us in \cite{LBCKKZ01,Li01,Li03}, as a normalized
version of the ``information metric'' of \cite{BGLVZ98,LiVi97}.
The mathematics used is based on Kolmogorov complexity theory \cite{LiVi97},
which is approximated using real-world compression software.
Roughly speaking, two objects are deemed close if
we can significantly ``compress'' one given the information
in the other, the idea being that if two pieces are more similar,
then we can more succinctly describe one given the other.
Let $Z(x)$ denote  the
binary length of the file $x$ compressed with compressor $Z$ (for example
"gzip", "bzip2", or "PPMZ"). The
{\em normalized compression distance} (\NCD) is defined as
\begin{equation}\label{eq.ncd}
NCD(x,y) = \frac{Z(xy) - \min \{Z(x),Z(y)\}}{\max \{Z(x),Z(y)\}},
\end{equation}
which is actually a family of distances parameterized with the
compressor $Z$.
The better $Z$ is, the
better the results are, \cite{CV04}.
This \NCD is used as distance $d$ in \eqref{eq.costs} to obtain the quartet
topology costs.

The \NCD in \eqref{eq.ncd} and a precursor 
have initially been applied to, among others,
alignment-free whole genome phylogeny, \cite{LBCKKZ01,Li01,Li03},
chain letter phylogeny \cite{BLM03},
constructing language trees \cite{Li03},
and plagiarism detection \cite{SID}.
It is in fact a parameter-free,
feature-free, data-mining tool.
A variant  has been experimentally tested on all time sequence data used in
all the major data-mining conferences in the last decade \cite{Ke04}.
That paper compared the compression-based method 
with all major methods used in
those conferences. The compression-based method
was clearly superior
for clustering heterogeneous data, and for
anomaly detection, and was competitive in clustering domain data.
The \NCD method turns out to be robust under change of the underlying
compressor-types: statistical (PPMZ), Lempel-Ziv based  dictionary (gzip),
block based (bzip2), or special purpose (Gencompress).
While there may be more appropriate special-purpose distance measures
for biological phylogeny, incorporating decades of research, the \NCD
is a robust objective platform to test the unbiased performance
of the competing phylogeny reconstruction algorithms.

\subsection{CompLearn Toolkit}\label{sect.complearn}
Oblivious to the problem area concerned, simply using the distances
according to the \NCD of \eqref{eq.ncd} and the 
derived quartet topology costs \eqref{eq.costs},
the MQTC heuristic described in Sections \ref{sect.mc}, \ref{sect.previous} 
fully automatically
clusters the objects concerned.
The method has been released in the public domain as open-source software:
The CompLearn Toolkit \cite{Ci03} is a suite
of simple utilities that one can use to apply compression
techniques to the process of discovering and learning patterns
in completely different domains, and hierarchically cluster them
using the MQTC heuristic.
In fact, CompLearn is so general that it requires
no background knowledge about any particular
subject area. There are no domain-specific parameters to set,
and only a handful of general settings. From CompLearn version 1.1.3
onwards the speedups and
improvements in Section~\ref{sect.previous} have been implemented.

\subsection{Previous Experiments}\label{sect.prevexp}
Using the CompLearn package, in
\cite{CV04} we studied hypotheses concerning mammalian evolution, by
reconstructing the phylogeny from
the mitochondrial genomes of 24 species.
These were downloaded from the
GenBank Database on the Internet.
In another experiment, we used the
mitochondrial genomes of molds and yeasts.
We clustered the SARS virus after its
sequenced genome was made publicly available,
 in relation to potentially similar viruses.
The \NCD distance matrix was computed using the compressor bzip2.
The resulting tree $T$ (with $S(T)=0.988$) was very similar to the
definitive tree based on medical-macrobio-genomics analysis,
appearing later in the New England Journal of Medicine,
\cite{SA03}.
In \cite{Ci07}, 
100 different H5N1 sample genomes were downloaded from the NCBI/NIH
database online, to analyze the geographical spreading
of the Bird Flu H5N1 Virus in a large example.

In general hierarchical clustering,
we constructed language trees,
cluster both Russian
authors in Russian, Russian authors in English translation, English
authors, handwritten digits given as two-dimensional OCR data,
and astronomical data. 
We also tested gross classification of files
based on heterogeneous data of markedly different file types:
genomes, novel excerpts, music files in MIDI format,  Linux x86 ELF executables,
and compiled Java class files, \cite{CV04}.
In \cite{CVW03},  MIDI data were used to cluster classical music,
distinguish between
genres like pop, rock, and classical, and do music classification.
In \cite{We05}, the CompLearn package was used to
analyze network traffic  and to cluster computer worms and viruses.
CompLearn was used to analyze medical clinical data
in clustering fetal heart rate tracings \cite{CBVA06}.
Other applications by different authors are in
software metrics and obfuscation, web page authorship,
topic and domain identification,
protein sequence/structure classification,
phylogenetic reconstruction, hurricane risk assessment,
ortholog detection,
and other topics.
Using code-word lengths
obtained from the page-hit counts returned by Google from the Internet,
we obtain a semantic distance between {\em names} for objects
(rather than the objects themselves) using the \NCD formula and viewing
Google as a compressor. 

Both the compression method and the Google method
have been used many times to obtain distances between objects
and to hierarchically cluster the data using CompLearn \cite{Ci03}. 
In this way, the MQTC method and heuristic described here has 
been used extensively. For instance, 
in many of the references in Google scholar
to \cite{CVW03,CV04,CV07}.
Here we give a first full and complete treatment of the MQTC problem,
the heuristic, speedup, and comparison to other methods.

\section{Comparing Against SplitsTree}\label{sect.splitstree}
We compared the performance of the MQTC heuristic as implemented in the
CompLearn package 
 against that of a leading application to 
compute phylogenetic trees, a 
program called SplitsTree~\cite{Hu06}.  Other methods include 
\cite{CL,SWR07,BaBe08}.
Our experiments
were initially performed with CompLearn version 0.9.7 
before the improvements in Section~\ref{sect.previous}. 
But with the improvements of 
Section~\ref{sect.previous} in CompLearn version 1.1.3 and later, 
sets of say 34 objects
terminated commonly in about 8 cpu seconds.
Below we use sets of 32 objects.
We choose SplitsTree version 4.6 for comparison and selected three tree
reconstruction methods to benchmark: NJ, BioNJ, and UPGMA. 
To make
comparison possible, we require a tree reconstruction implementation
that takes a distance matrix as input. This requirement ruled
out some other possibilities, and motivated our choice. 
To score the quality of the trees produced by CompLearn
and SplitsTree we converted the SplitsTree
output tree to the CompLearn output format.
Then we used the $S(T)$ values in the CompLearn output and
the converted SplitsTree output to compare the two.
The quartet topology costs were derived from the distance matrix concerned
as in Section~\ref{sect.previous}.

The UPGMA method
consistently
performed worse than the other two methods in SplitsTree. 
In several trials it failed
to produce an answer at all (throwing an unhandled Java Exception), which
may be due to an implementation problem. Therefore,
attention was focused on the other two methods.
Both NJ \cite{SN87} and BioNJ \cite{BN97} are neighbor-joining methods.
In all tested cases they produced the same trees, therefore
we will treat them as the same (SplitsTree BioNJ=NJ) in this discussion.

Our MQTC heuristic  has through the Complearn
package already been extensively tested
in hierarchical clustering of nontree-structured
data as reviewed in Section~\ref{sect.prevexp}.
Therefore, we choose to run the MQTC heuristic 
and SplitsTree on  data favoring SplitsTree, that is, tree-structured data,
both artificial and natural.

\subsection{Testing on Artificial Data 100 Times}\label{sect.artificial}
We first test whether the MQTC heuristic
and the SplitsTree methods are trustworthy.
We generated 100 random samples of
an unrooted binary tree $T$ with 32 leaves as follows:
We started with a linear tree with each internal node connected
to one leaf node, a prior internal node, and a successive internal node.
The ends have two leaf nodes instead.  This initial tree was then
mutated 1000 times using randomly generated instances of the
complex mutation operation defined earlier.
Next, we derived a
metric from the scrambled tree by defining the distance between
two nodes as follows:
Given the length of the path from $a$ to $b$ in an integer number of
edges as $L(a,b)$, let
\[d(a,b) = { {L(a,b)+1} \over 32},
\]
  except when
$a = b$, in which case $d(a,b) = 0$.  It is easy to verify that this
simple formula always gives a number between 0 and 1, is monotonic
with path length, and the resulting matrix is symmetric.
Given only the $32\times 32$ matrix of these normalized distances,
our quartet method precisely reconstructed the original tree one hundred times
out of one hundred random trials. Similarly, SplitsTree NJ and
BioNJ also reconstructed each tree precisely in all trials. However UPGMA was
unable to cope with this
test.  It appears there is a mismatch of assumptions in this experimental
ensemble and the UPGMA preconditions, or there may be an error in the
SplitsTree implementation. The running time of CompLearn  without
the improvement of Section~\ref{sect.previous} was about 3 hours
per example, but with the improvement of Section~\ref{sect.previous} only
at most 5 seconds per example.  
SplitsTree had a similar but slightly higher running time.
Since the performance of CompLearn and SplitsTree (both NJ and BioNJ)
was 100\% correct on the artificial data we feel that all the methods 
except SplitsTree UPGMA perform
satisfactory on artificial tree-structured data.

\begin{figure*}
\begin{small}
\begin{tabular}{|l|l||l|l|}
\hline
Acipenser dabryanus & Yangtze sturgeon fish & Lipotes vexillifer & Yangtze river dolphin \\
Amia calva & Bowfin fish & Melanogrammus aeglefinus & Haddock \\
Anguilla japonica & Japanese eel & Metaseiulus occidentalis & Western predatory mite \\
Anopheles funestus & Mosquito & Neolamprologus brichard &  Lyretail cichlid fish \\
Arctoscopus japonicus & Sailfin sandfish & Nephila clavata & Orb web spider \\
Asterias amurensis & Northern Pacific seastar & Oreochromis mossambicus & Mozambique tilapia fish \\
Astronotus ocellatus & Tiger oscar & Oscarella carmela & Sponge \\
Cervus nippon taiouanus & Formosan sika deer & Phacochoerus africanus & Warthog \\
Cobitis sinensis & Siberian spiny loach fish & Plasmodium knowlesi & Primate malaria parasite \\
Diphyllobothrium latum & Broad tapeworm & Plasmodium vivax & Tersian malaria parasite \\
Drosophila melanogaster & Fruit fly & Polypterus ornatipinnis & Ornate bichir fish \\
Engraulis japonicus & Japanese anchovy & Psephurus gladius & Chinese paddlefish \\
Gavia stellata & Red throated diver & Pterodroma brevirostris & Kerguelen petrel \\
Gymnogobius petschiliensis & Floating goby fish & Savalia savaglia & Encrusting anemone \\
Gymnothorax kidako & Moray eel & Schistosoma haematobium & Vesical blood fluke \\
Hexamermis agrotis & Roundworm Nematode & Schistosoma spindale & Cattle fluke \\
Hexatrygon bickelli & Sixgill stingray & Synodus variegatus & Variegated lizardfish \\
Homo sapiens & Human & Theragra finnmarchica & Norwegian pollock fish \\
Hynobius arisanensis & Arisian salamander & Tigriopus californicus & Tidepool copepod \\
Hynobius formosanus & Formosa salamander & Tropheus duboisi & White spotted cichlid fish \\
Lepeophtheirus salmonis & Sea lice & & \\

\hline
\end{tabular}
\end{small}
\caption{Listing of scientific and corresponding common names of 41 (out of 45) species used. 
The remaining four are dogs, with common breed names
Chinese Crested,
Irish Setter,
Old English Sheepdog,
Saint Bernard. There are no scientific names distinguishing them, 
as far as we know.}
\label{figanimtab}
\end{figure*}

\subsection{Testing on Natural Data 100 Times}
\label{sect.nat}
In the biological setting
the data are (parts of) genomes of currently existing species,
and the purpose is to reconstruct the evolutionary tree that led
to those species. Thus, the species are labels of the leaves,
and the tree is traditionally binary branching with each branching
representing a split in lineages. The internal nodes and the root
of the tree correspond with extinct species (possibly a still
existing species in a leaf directly connected to the internal node).
The root of the tree is commonly
determined by adding an object that is known to be less related
to all other objects than the original objects are with respect to
each other. Where the unrelated object joins the tree is where
we put the root.
In this setting, the direction from the root to the leaves represents
an evolution in time, and the assumption is that there is a true
tree we have to discover.

However, we can also use the method for hierarchical clustering,
resulting in an unrooted ternary tree.
The interpretation is that objects in a given subtree are pairwise
closer (more similar) to each other than any of those objects
is with respect to any object in a disjoint subtree.

To evaluate the quality of tree reconstruction for natural
genomic data, we downloaded 45 mitochondrial
gene sequences, Figure~\ref{figanimtab},
and randomly selected 100 subsets of 32 species each.  We
used CompLearn with PPMD to compute \NCD matrices for each of the 100 trials
and fed these matrices (as Nexus files) to both CompLearn and SplitsTree.
CompLearn without the speedup in Section~\ref{sect.previous} 
took about 10 hours per tree, 
but with the speedup of  Section~\ref{sect.previous} CompLearn
takes at most 6 seconds for collections of 32 objects in 66\% of the cases,
at most 10 seconds in 90\% of the cases, and occasionally (about 10\%
of the cases)
between 10 seconds and 2 minutes.
SplitsTree used about 10 seconds per trial.
In all but one case out of 100 trials, 
CompLearn performed as good or better in the sense of
producing trees with an as good or higher S(T) score than the best method
(with UPGMA performing badly and NJ and BioNJ giving the same scores) 
from SplitsTree. The results
are shown in the histogram Figure~\ref{fig.nathisto}, which shows that 
out of 100 trials CompLearn produced a better tree in 69\% of the trials.
\begin{figure}[htb]
\begin{center}
\epsfig{figure=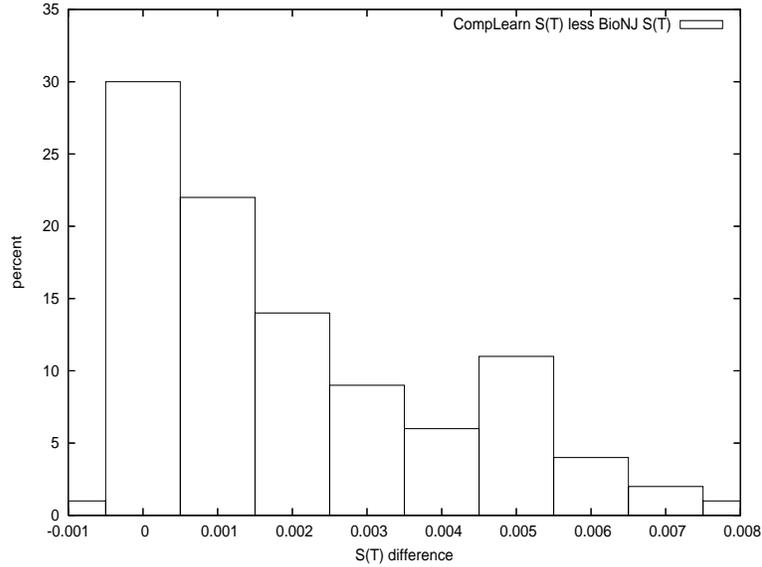,width=3in,height=4in,angle=-90} 
\caption{Histogram showing CompLearn S(T) advantage over SplitsTree S(T)}
\label{fig.nathisto}
\end{center}
\end{figure}
CompLearn had an {\em average} $S(T)$ of 0.99487068.  
SplitsTree achieved the {\em best}
$S(T)$ with both NJ and BioNJ at 0.99243944.  At this high level the absolute
magnitude of the difference is small, yet it can still imply significant
\begin{figure*}
\begin{center}
\epsfig{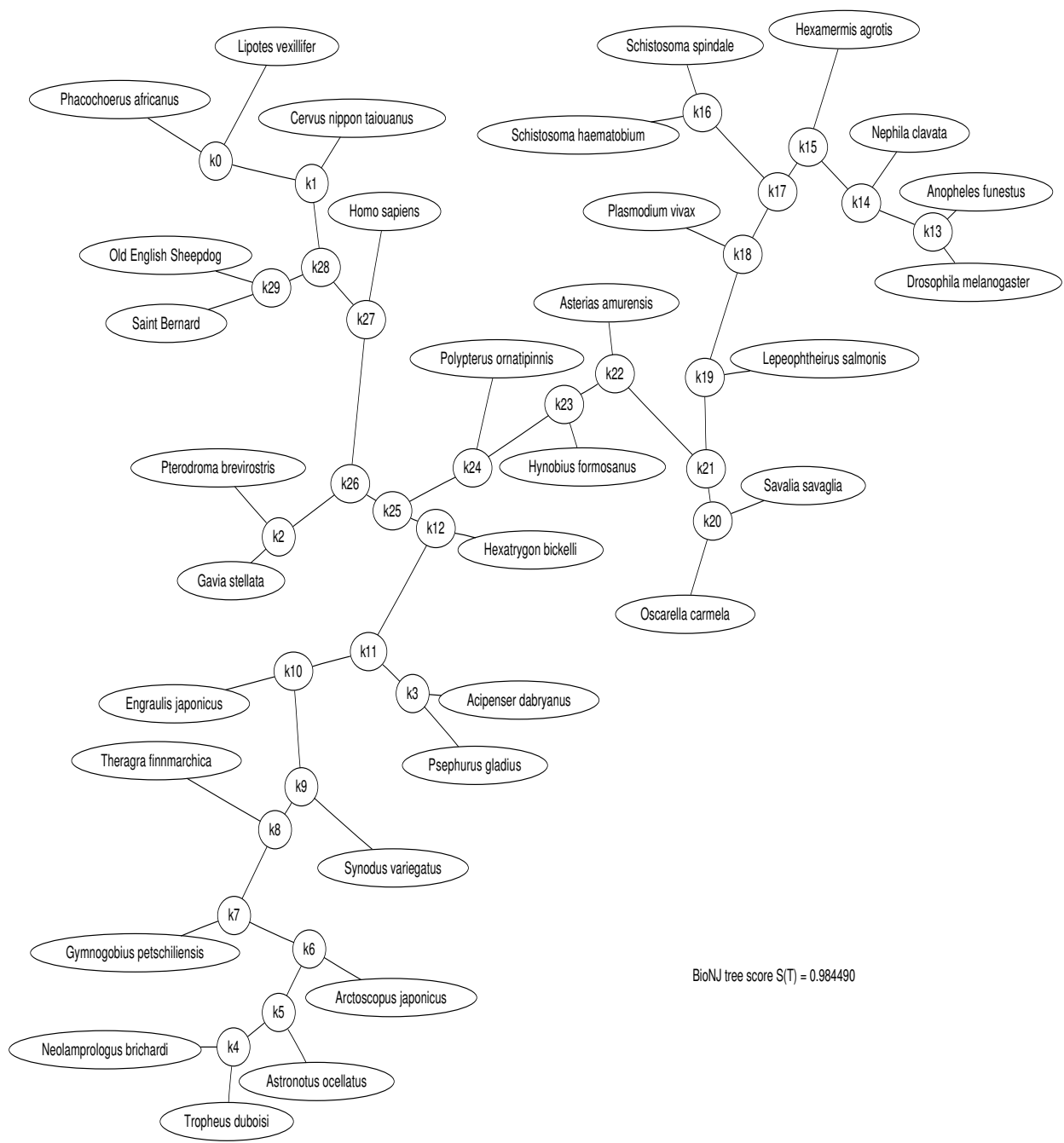} 
\caption{NJ=BioNJ tree from SplitsTree}
\label{fig.bt}
\end{center}
\end{figure*}
changes in the structure of the tree.  Figure~\ref{fig.bt} and Figure~\ref{fig.clt} depict one example showing both BioNJ=NJ and CompLearn trees applied to
the same input matrix from one of the natural data test cases described
above.  In this case there are important differences in placement of at
least two species; {\em Hexatrygon bickelli} and {\em Synodus variegatus}.

\begin{figure*}
\begin{center}
\epsfig{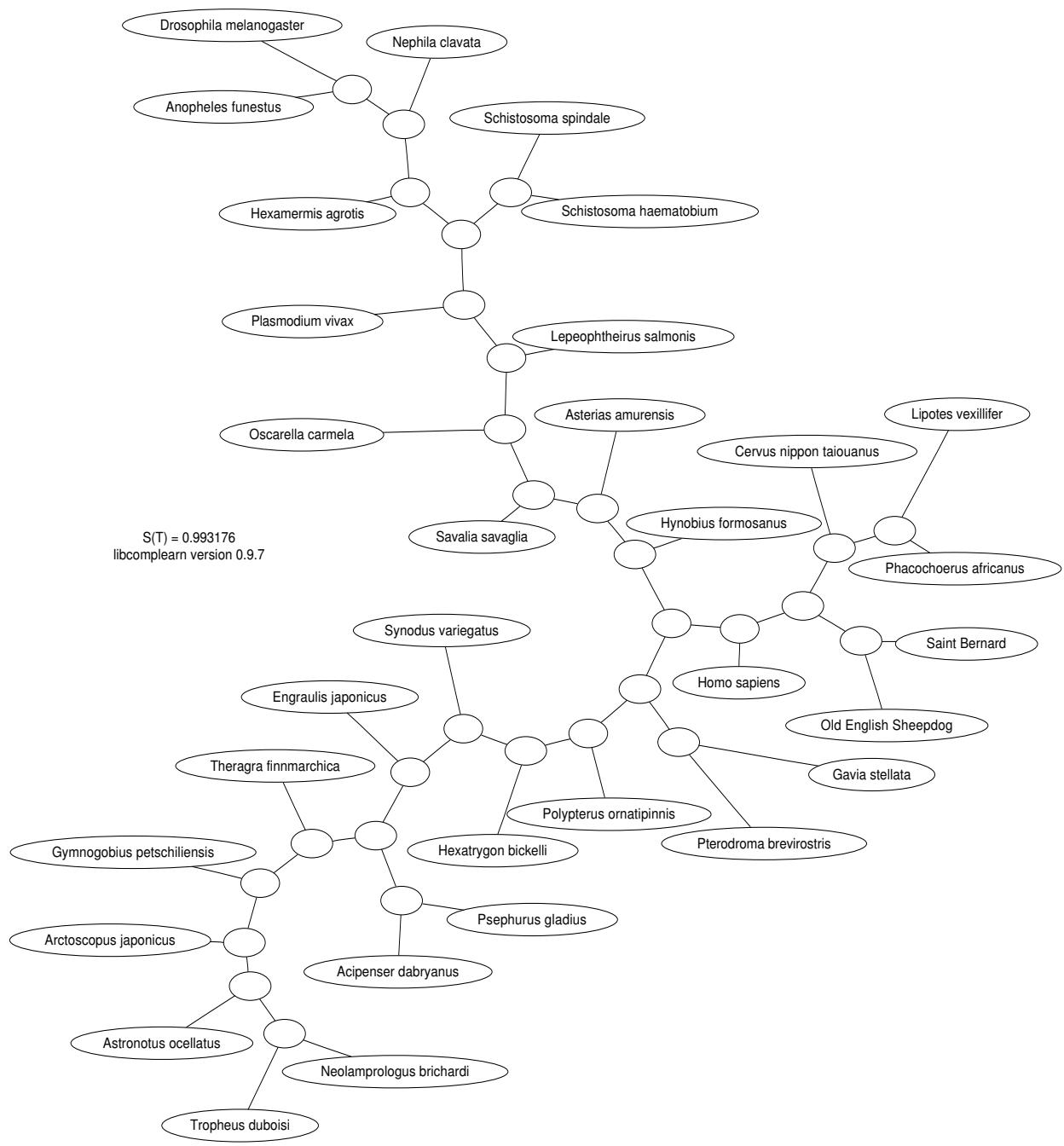} 
\caption{CompLearn tree for comparison with previous Figure}
\label{fig.clt}
\end{center}
\end{figure*}

Although we can not know for sure the true maximum value that can be attained
for the $S(T)$, given an arbitrary distance matrix, we can still define a useful
quantity. Let
\begin{equation}\label{eq.rt}
R(T) = 1.0 - S(T),
\end{equation}
and term $R(T)$ the {\em room for improvement} for tree $T$. This is especially
apt in cases like the present one where we know that the 
optimal tree $T_{\rm opt}$
has $S(T_{\rm opt})$
close to 1. Suppose CompLearn produces tree $T$ in trial $t$
and SplitsTree produces tree $T'$ in trial $t$. Define
$R_C(t)= R(T)$ 
and 
$R_S(t)=R(T')$ using \eqref{eq.rt}.
We can compute the decibel gain $db(t)$ as the logarithm of the ratio of
room for improvement in trial $t$ of
SplitsTree's answer versus CompLearn's answer with the formula
\begin{equation}\label{eq.db}
 db(t) = 10 \log_{10} \frac{R_S(T)}{R_C(T)}. 
\end{equation}
Hence if $db(t)=1$ then $R_S(T) =10^{1/10} R_C(T) \approx 1.3 R_C(T)$, 
and $db(t)=2$ means that $R_S(T) =10^{1/5} R_C(T) \approx 1.6 R_C(T)$. 
This is statistically significant according to almost every reasonable criterion.
Note that the room for improvement decibel gain $db(t)$ 
in \eqref{eq.db} represents also a conservative
estimate of the true improvement decibel gain in real error terms.
This is because the true
maximum $S(T)$ score of a tree $T$ resulting
from a distance matrix is always less than or equal to 1. Using the
$S(T_{\rm opt})$ value of the real optimal tree $T_{\rm opt}$ instead of 1 would
only make the gain more extreme.  We plot the decibel room for improvement
gain in Figure~\ref{fig.edeci}, using different binning 
boundaries than in Figure
\ref{fig.nathisto}. On the horizontal axis the bins are
displayed where for every trial $t$ we put $db(t)$ in the appropriate bin. 
On the vertical
axis the percentage of the number of elements
in a particular bin to the total is depicted.

Because now we use different boundaries for each bin, the percentage
of trials
with the same room for improvement for both CompLearn
and SplitsTree is slightly higher than the
percentage of trials with the same $S(T)$ values between CompLearn
and SplitsTree in Figure~\ref{fig.nathisto}. Yet now we can see 
the important difference in room for improvement between CompLearn
and SplitsTree expressed in decibels.
Thus, about 38\% of the CompLearn trials gives no positive integer decibel
reduction in room for improvement
over the SplitsTree performance (and 1\% gave a negative reduction). 
About 27\% gives a 1db reduction
in room for improvement,
about 22\% gives a 2db reduction in room for improvement, 
about 10\% gives a 3db in room for improvement.
Overall, about 61\% of the CompLearn trials
gives a  1 or more decibel reduction in  room for 
improvement over the SplitsTree performance.
In more than 1/3 of the
trials CompLearn achieves at least a 2db reduction in room for improvement as
compared to SplitsTree.
\begin{figure}
\begin{center}
\epsfig{figure=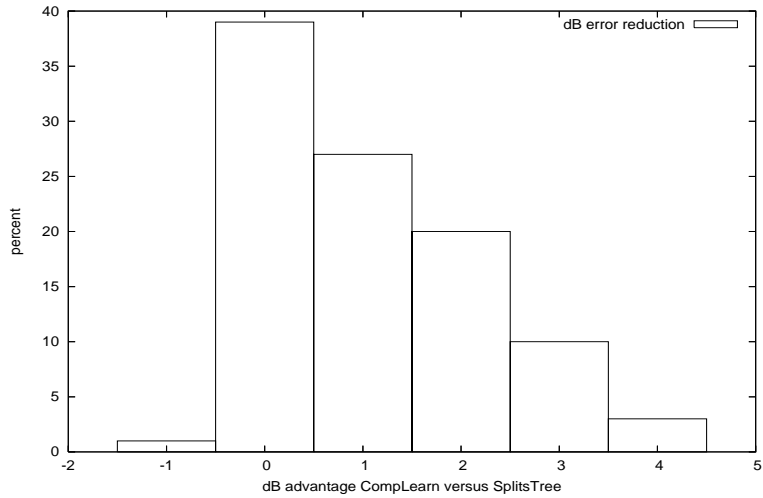,width=5in,height=5in,angle=0} 
\caption{Decibel error reduction from CompLearn}
\label{fig.edeci}
\end{center}
\end{figure}

\section{Conclusion}
We have introduced a new quartet tree problem, 
the Minimum Quartet Tree
Cost (MQTC) problem, suited for general
hierarchical clustering. This new method relies on global optimization
of the constructed tree in contrast to bottom-up or top-down
methods that can get stuck
in local optima, such as Quartet Puzzling, neighbor joining, and the like.
 Is is shown that this MQTC problem is 
NP-hard by a reduction to the (weighted) Maximum Quartet Consistency
(MQC) problem that is more suited for the restricted
case of biological phylogeny.
Moreover, if there is a polynomial time approximation scheme (PTAS)
for the MQTC optimization problem, then P=NP.
Given the hardness of the MQTC problem we introduce a Monte Carlo
heuristic based on randomized hill climbing. This heuristic runs in 
time which is theoretically $\Omega(n^4)$ per generation 
where $n$ is the number of objects, 
and $O(n^5)$ per generation in the implemented version.
The improvement
in Section~\ref{sect.previous} based on the distance matrix and
quartet topology costs in \eqref{eq.costs} runs in time $O(n^3)$ per generation both 
as algorithm and implementation.
The new method including the
improvement is available 
for general use in the open software CompLearn Toolkit \cite{Ci03} from
version 1.1.3 onward.
It has been used widely for general hierarchical clustering
and also for biological phylogeny. Here, we tested our MQTC heuristic
on artificial data and natural data, and compared it with the neighbor-joining
method available in the (highly competitive) 
SplitsTree package (version 4. 6) designed for tree-structured
data in Biological Phylogeny. 
(BioNJ and NJ in the SplitsTree package always gave the same
results in our experiments, so we treat them as one,
and the UPMG method in the SplitsTree
package did
not work for us.)
To make the comparison more disadvantageous to our MQTC heuristic
 and more
advantageous to SplitsTree we tested it on tree-structured data, rather than
general hierarchical clustering on data of unknown structure.
SplitsTree was generally slower 
(sometimes 10 seconds versus 6 seconds for 
CompLearn version 1.1.3 and later, 
that is, 2/3rd more)
than our MQTC heuristic with the
improvements in Section~\ref{sect.previous}. 
On our artificial data experiments both our MQTC heuristic and the SplitsTree methods
gave 100\% correct results. On the natural data experiments the 
{\em average case}
of our MQTC heuristic was better than the {\em best case} of the SplitsTree
heuristics. 
To amplify the differences we compared the decibel gain in room for
improvement of SplitsTree's answers versus our MQTC heuristic's answers.
In 61\% of the trials our MQTC heuristic's performance gave a  positive
integer decibel reduction in room for improvement
over SplitsTree's performance, and in 33\% of the trials 
our MQTC heuristic's performance gave a 2db reduction in room for improvement
over SplitsTree.
Other heuristics
for the MQTC optimization problem are recently given in \cite{CDGKP08}.
But even the best method in \cite{CDGKP08} has
a slower running time for natural data (with  $n=32$ typically about 50\%)
than the implementation of the MQTC heuristic in 
CompLearn from version 1.1.3 onward.

\appendix
\subsection{Sufficiency of the Set of Simple Mutations}
\label{sect.mutations}

\begin{theorem}\label{theo.mutations}
Every  ternary tree with $n$ leaves
labeled $l_1,l_2, \ldots , l_n$
and $n-2$ unlabeled internal nodes can be
transformed in every other 
ternary tree with $n$ leaves
labeled $l_1,l_2, \ldots , l_n$ and $n-2$ unlabeled internal nodes
by a sequence of $f(n)$ mutations consisting of subtree to leaf swaps
or leaf to leaf swaps where
$f(3) \leq 3$ and 
$f(n) \leq 5n-16$ for $n \geq 4$.
\end{theorem}
\begin{proof}
For convenience of
the discussion we attach labels to
the internal nodes, but actually the internal nodes are unlabeled,
only the leaves are labeled..
%
The proof is by induction on the number of nodes. 

{\em Base case: $n=3$.} There is one internal node, so the theorem is vacuously
true. For $n=4$ there are two internal nodes, so the theorem is
true as well using at most one leaf swap.

{\em Induction}: Assume the theorem is correct for every $k$ with 
$4 \leq k < n$. We prove that it holds for $k=n$.
For $n>4$ consider a ternary tree $T_0$ with $n-2$ unlabeled internal nodes 
$1,2,\ldots,(n-2)$ and $n$ labeled leaves that
has to be transformed into a ternary tree $T_1$ with the same
unlabeled internal nodes and labeled leaves. 

Assume that the initial tree $T_0$ has a path $z-x-y$ where $y$ is
an end internal node with two leaves and $x$ is an internal node with
one leaf. If $T_0$ is not of that form then we make it of that
form by a subtree to leaf swap: Take another end internal node $u$
(possibly $z$) and swap the 3-node tree rooted at $u$ ($u$ and its
two leaves) with a leaf of $y$.
This results in a path $x-y-u$ where $u$ is an end internal
node with two leaves and $y$ is an internal node with a single
leaf. We start from the resulting tree which we call $T_0$ now.
 
For the sake of the argument we number the nodes so that 
$n-2$ is an end internal
node connected to an internal node $n-3$ which has
a single leaf $l$.
Glue the internal nodes $n-3$ and $n-2$ and
the leaf $l$ together
in a single internal node
now denoted as $n-3$. The new $n-3$ is an end internal
node connected to two leaves formerly connected to the
old $n-2$.
This results in an $n-3$ unlabeled 
internal node ternary tree $T'_0$ with $n-1$ labeled leaves. 

By the induction assumption  we can
transform $T'_0$ into any ternary tree $T'_1$ with 
$n-3$ unlabeled internal nodes and $n-1$ labeled leaves in $f(n-1)$
subtree to leaf or leaf to leaf mutations. Take $T'_1$ to be
a subtree of 
$T_1$ with the following exceptions. 
Since $T_1$ has one more internal node
than $T_1'$ 
we can choose that extra internal node as an end internal node
attached to tree $T_1'$ at the place where there is now a leaf.
Let that leaf be leaf $l'$. 
Note that leaf $l$ is not a leaf of $T_1'$ (since
it is incorporated in internal node $n-3$). If $l$
should be a leaf of $T'_1$ to make it a subtree of $T_1$ 
then
we swap $l$ with the leaf $l''$ in the place where $l$
has to go in a leaf-to-leaf swap. For convenience we still
denote the leaf
left in the composite node $n-3$ by $l$.

Now expand in $T'_1$ the
internal node $n-3$ into the path $(n-3)-(n-2)$ 
together with leaf $l$ connected
to $(n-3)$. This yields a ternary tree with $n-2$ unlabeled internal 
nodes and $n$ labeled leaves.
There are three cases. 

{\bf Case 1.} Initially, in $T'_1$ the node $n-3$ is an 
end internal node connected to 
an internal node  $u$ as in the path $u-(n-3)$. The expansion takes us to
the situation that we have a path $u-(n-3)-(n-2)$ and leaf $l$ connected
to $n-3$.

The old internal node $n-3$ being an end internal
node had two leaves. In the path $(n-3)-(n-2)$
both these leaves stay connected to the new end internal node $n-2$
and leaf $l$ stays connected to $n-3$. 

Assume first that $l'$ is not
in de 5-node subtree rooted at the new $n-3$ containing the path $(n-3)-(n-2)$. 
We interchange this 5-node subtree
with
the leaf $l'$.
Next, we interchange
the 3-node subtree rooted at $n-2$ with the leaf $l'$ at its new location.
In this way, $n-3$ being now in the former
position of leaf $l'$ is the missing internal node of $T_1$.
The new internal node $n-3$ is an end internal node with
 two leaves $l,l'$ of which
$l'$ is in the correct position. There is still
the leaf $l''$ being possibly in the wrong position.
All the other leaves are in the correct position for $T_1$.
After we swap the leaves $l,l''$ if necessary, all leaves
are in the correct position.

Assume second that $l'$ is 
in the subtree rooted at $n-3$ containing the path $(n-3)-(n-2)$.
Then, the new $n-2$ being an end internal node in the former
position of leaf $l'$ is the missing internal node of $T_1$.

The total number of mutations used is at most three consisting of two
subtree to leaf swaps and possibly one leaf to leaf swap.

{\bf Case 2.} Initially, the node $n-3$ in $T_1'$ is connected to two
internal nodes yielding a path $u-(n-3)-v$ such that $n-3$ is connected also to
one leaf, say $l'''$.  The expansion takes us to
the situation that we have a path $u-(n-3)-(n-2)-v$.
The old internal node $n-3$ was connected to leaf $l'''$ 
which leaf is now connected to $n-2$. The leaf $l$
is still connected to the new $n-3$.

Assume first that $l'$ is not in the subtree 
rooted at the new $n-3$ (containing $u-(n-3)$). 
We interchange the subtree rooted at $n-3$ (containing $u-(n-3)$)  with leaf $l'$.
Next we interchange the subtree rooted at $u$ (not containing
$n-3$ and leaf $l$) with $l'$ again. Now the new $n-3$ takes the place of the missing
internal node of $T_1$ and it is an end internal node connected
to leaves $l,l'$. Of these, leaf $l'$ is in correct position. 
All the other leaves except possibly $l,l''$ are in correct position
for $T_1$. If necessary we interchange leaves $l,l''$.

Assume second that $l'$ is in the subtree 
rooted at the new $n-3$ (containing $u-(n-3)$).
Interchange the subtree rooted at $n-3$ (containing $(n-3)-(n-2)$)  
with leaf $l'$. Next we interchange the subtree rooted at $n-2$ (not containing
$n-3$ and leaf $l$) with $l'$ again.
Now the new $n-3$ takes the place of the missing
internal node of $T_1$ and it is an end internal node connected
to leaves $l,l'$. Of these, leaf $l'$ is in correct position.
All the other leaves except possibly $l,l''$ are in correct position
for $T_1$. If necessary we interchange leaves $l,l''$.

The total number of mutations used is at most three consisting of two
subtree to leaf swaps and possibly one leaf to leaf swap.

{\bf Case 3.}  Initially, in $T_1'$ the node $n-3$ is connected to three
internal nodes forming the path $u-(n-3)-v$ and there is a path $w-(n-3)$
with $w \neq u,v$.
The expansion yields the path $u-(n-3)-(n-2)-v$ with leaf $l$ connected to
$n-3$ and $n-2$ is also in a path $w-(n-2)$.

Assume first that $l'$ is not in the subtree
rooted at the new $n-3$ (containing $u-(n-3)$).
We interchange the subtree rooted at the new $n-3$ containing the edge
$u-(n-3)$ and leaf $l$ with the leaf $l'$. 
Subsequently, we interchange
the subtree rooted at $u$ (not containing $n-3$ and the connected
leaf $l$) with $l'$ again. Now node $n-3$ is in the position of
the missing internal node of $T_1'$
and it is an end internal node with two leaves $l,l'$.
Of these, $l'$ is in correct position. Moreover, all the other
leaves are in correct position except possibly $l,l''$.
 If necessary we interchange leaves $l,l''$.

Assume second that $l'$ is in the subtree
rooted at the new $n-3$ (containing $u-(n-3)$).
Interchange the subtree rooted at the new $n-3$ containing the edge
$(n-3)-(n-2)$ and leaf $l$ with the leaf $l'$.
Subsequently, we interchange
the subtree rooted at $n-2$ (not containing $n-3$ and the connected
leaf $l$) with $l'$ again. Now node $n-3$ is in the position of
the missing internal node of $T_1'$
and it is an end internal node with two leaves $l,l'$. 
Of these, $l'$ is in correct position. Moreover, all the other
leaves are in correct position except possibly $l,l''$.
 If necessary we interchange leaves $l,l''$.

The total number of mutations used is at most three consisting of two
subtree to leaf swaps and possibly one leaf to leaf swap.

We count the number of mutations as follows. Initially, tree $T_1'$
required at most $f(n-1)$ mutations to be obtained from tree $T_0'$.
By the above analysis $f(n) \leq f(n-1)+5$ (remember the possibly
necessary initial subtree to leaf swap to bring $T_0$ in t
he required form, and the possibly ncessary leaf to leaf swap
between $l$ comprised in the composite node $n-3$ and $l''$
just before {\bf Case 1}).
The base case shows that $f(3) \leq 3$ and $f(4) \leq 4$.
Hence, $f(n) \leq 4+(n-4)5=5n-16$ for $n > 4$.
\end{proof}
\begin{remark}
\rm
Note that the only mutations used are leaf-to-leaf swaps and subtree
to leaf swaps. This shows that the other mutations, that is subtree
to subtree swaps, and subtree transfers are superfluous 
in terms of completeness. However, they may considerably reduce the number
of total mutations required to go from one tree to another. 
Using the full set of mutations we believe it is possible to go from
a ternary tree as above to another one
in at most $n$ mutations as given.
\end{remark}
 
\section*{Acknowledgement}
We thank Maarten Keijzer for the improvements in the heuristic and its 
implementation described in Section~\ref{sect.previous}.

\begin{small}

\end{small}

\end{document}